\documentclass[11pt]{article}

\usepackage{latexsym,mathrsfs}
\usepackage{amsmath,amssymb} 
\usepackage{amsthm,enumerate,verbatim}
\usepackage{amsfonts}
\usepackage{graphicx}
\usepackage{algorithm}
\usepackage{algorithmic}
\usepackage{url}
\usepackage[dvipsnames]{xcolor}
\usepackage{pifont}
\usepackage{mathtools}
\usepackage{makecell}
\usepackage[normalem]{ulem}
\usepackage{hyperref}

\DeclarePairedDelimiter\floor{\lfloor}{\rfloor}

\newtheorem{lemma}{Lemma}

\newtheorem{theorem}{Theorem}

\newtheorem*{definition*}{Definition}
\newtheorem{example}{Example}

\newtheorem{assumption}{Assumption}
\newtheorem{remark}{Remark}

\DeclareMathOperator{\conv}{conv}

\DeclareMathOperator{\cone}{cone} 
\DeclareMathOperator{\rank}{rank}

\DeclareMathOperator{\logdet}{logdet}

\newcommand\undermat[2]{%
	\makebox[0pt][l]{$\smash{\underbrace{\phantom{%
					\begin{matrix}#2\end{matrix}}}_{\text{$#1$}}}$}#2}

\definecolor{brightpink}{rgb}{1.0, 0.0, 0.5}
\newcommand{\ngc}[1]{{\color{brightpink} (\textbf{NG:} #1)}}

\title{Simplex-Structured Matrix Factorization: Sparsity-based  
Identifiability and Provably Correct Algorithms} 
\date{}

\author{Maryam Abdolali$^{*}$ \qquad Nicolas Gillis\thanks{Emails: \{maryam.abdolali, nicolas.gillis\}@umons.ac.be. The authors acknowledge the support by the European Research Council (ERC starting grant no 679515), and by the Fonds de la Recherche Scientifique - FNRS and the Fonds Wetenschappelijk Onderzoek - Vlaanderen (FWO) under EOS Project no O005318F-RG47.} \\   
	Department of Mathematics and Operational Research \\ 
	Facult\'e Polytechnique, Universit\'e de Mons \\ 
	Rue de Houdain 9, 7000 Mons, Belgium
}

\begin{document}

\maketitle

\begin{abstract}
In this paper, we provide novel algorithms with identifiability guarantees for simplex-structured matrix factorization (SSMF), a generalization of nonnegative matrix factorization. Current state-of-the-art algorithms that provide identifiability results for SSMF rely on the sufficiently scattered condition (SSC) which requires the data points to be well spread within the convex hull of the basis vectors. The conditions under which our proposed algorithms recover the unique decomposition is in most cases much weaker than the SSC.  We only require to have $d$ points on each facet of the convex hull of the basis vectors whose dimension is $d-1$. The key idea is based on extracting facets containing the largest number of points.  We illustrate the effectiveness of our approach on synthetic data sets and hyperspectral images, showing that it outperforms state-of-the-art SSMF algorithms as it is able to handle higher noise levels, rank deficient matrices, outliers, and input data that highly violates the SSC. 
\end{abstract}



\section{Introduction}


Extracting meaningful underlying structures that are present in high-dimensional data sets is a key problem in machine learning, data mining, and signal processing. 
 Structured matrix factorization (SMF) is a general model for exploiting latent linear structures from data; see for example~\cite{UHZB14, fu2020nonconvex} and the references therein. 
Given a factorization rank $r$, 
SMF expresses the input matrix $X\in \mathbb{R}^{m \times n}$ as the product of two matrices $W \in \mathbb{R}^ {m \times r}$ and $H \in \mathbb{R}^{r \times n}$, with some restrictions on the structure of $W$ and/or $H$. 
 This paper focuses on a specific SMF model called simplex-structured matrix factorization (SSMF). 

\paragraph{Simplex-structured matrix factorization} 
Given an $m$-by-$n$ matrix $X$ (with $m$ dimensional data points as columns) and an integer $r$, SSMF looks for an $m$-by-$r$ matrix $W$ whose columns are the basis vectors, and an $r$-by-$n$ matrix $H$ containing the mixing weights such that $X \approx WH$ and with the property that each column of $H$ belongs to the unit simplex, that is, $H(:,j) \in \Delta^r$ for all $j$ where 
\[
\Delta^r 
= \left\{ x \in \mathbb{R}^r \ \Big| \ x \geq 0, \sum_{i=1}^r x_i = 1 \right\} . 
\] 
In the exact case, that is, $X=WH$, each column of $X$ belongs to the convex hull generated by the columns of $W$, that is, 
\[
\conv(X) \quad \subseteq \quad \conv(W), 
\] 
where $\conv(X)  =  \{ x \ | \ x = Xh, h \in \Delta^n \}$. 
SSMF is a generalization of nonnegative matrix factorization (NMF), an SMF problem where $W$ and $H$ are required to be nonnegative, 
while $X$ is nonnegative as well. 
The main advantage of NMF over other SMFs such as the PCA/SVD is its interpretability when the factors $W$ and $H$ have a physical meaning; see~\cite{cichocki2009nonnegative, gillis2014and, fu2018nonnegative} and the references therein.  
In the exact case, NMF can be formulated as an SSMF problem using a simple scaling of the columns of $X$ and $W$. 
In fact,   
defining\footnote{We assume that the columns of $X$ and $W$ are different from zero otherwise they can be discarded.} 
$D_X$ as the diagonal matrix with $(D_X)_{ii} = ||X(:,i)||_1$ for all~$i$, we have 
\[
\underbrace{X (D_X)^{-1}}_{X'} 
\quad = \quad 
\underbrace{ W (D_W)^{-1} }_{W'} 
\; 
\underbrace{ D_W H (D_X)^{-1} }_{H'} . 
\] 
Since the entries of each column of $X'$ and $W'$ sum to one, 
and since $X'(:,j) = W'H'(:,j)$ for all $j$, 
the entries of the columns of $H'$ must also sum to one, that is, $H'(:,j) \in \Delta^r$ for all $j$. In fact, letting $e$ be the  vector of all ones of appropriate dimension, we have $e^\top = e^\top X'  = e^\top W'H' = e^\top H'$. 
Note that SSMF is a constrained variant of semi-NMF which only requires  the factor $H$ to be nonnegative; see~\cite{gillis2015exact} and the references therein.

\paragraph{Applications}

Let us discuss in more details two major applications of SSMF: blind hyperspectral unmixing,  and topic modeling. 
We refer the interested reader to~\cite{Wu2017ASM} and the references therein for more applications.  

A hyperspectral image is a data cube that consists of hundreds of two dimensional spatial images that are acquired at different contiguous wavelengths (known as spectral bands). In other terms, for each spatial location, that is, for each pixel, a hyperspectral image records a so-called spectral signature (the intensity of light depending on the wavelength).   
These images have a vast variety of applications in remote sensing, military surveillance, and environmental monitoring \cite{matteoli2010tutorial}. Due to the limited spatial resolution of hyperspectral sensors, a pixel may contain a mixture of light radiance from the materials located in the captured scene. This led to the development of a family of algorithms for extracting pure materials that are present in the image, known as \textit{endmembers}, and specifying the \textit{abundance} of these materials in each pixel. These algorithms are known as hyperspectral unmixing (HU). Under the linear mixing assumption, HU can be modeled as an SSMF problem; see~\cite{bioucas2012hyperspectral, ma2013signal} and the references therein. 
Constructing the matrix $X$ by stacking the spectral signature of the pixels as its columns, each column of $W$ can be interpreted as the spectral signature of an endmember, and each column of the matrix $H$ represents the abundance of the endmembers in the corresponding pixel. 

Another essential application of SSMF is text mining \cite{arora2012learning, huang2018learning, huang2016anchor}. A commonly used approach for modeling documents in natural language processing is the bag-of-words model where each document is represented by a vector that contains the frequency of occurrences of a predefined set of words~\cite{wallach2006topic}. 
Hence, a collection of documents form a matrix $X$ where the $(i,j)$th element indicates the frequency of the $i$th word in the $j$th document. Extracting latent topic patterns across the documents and categorizing the documents according to the extracted topics is an essential task when processing textual information. By applying SSMF on the given document matrix, it is decomposed as the product $WH$ of two matrices, where each column of $W$ can be interpreted as a hidden topic,  and each column of $H$ can be regarded as the proportion of the topics discussed in the corresponding document.

\paragraph{Identifiability} In many applications, a crucial question about SSMF is when the factors $W$ and $H$ can be uniquely recovered. 
SSMF never has a unique solution, unless some additional constraints are imposed on the factors $W$ and/or $H$. In fact, if there exists a polytope $\conv(W)$ containing the columns of $X$, then any larger polytope containing $\conv(W)$ leads to another solution of SSMF. 
Suppose $X$ is generated by multiplying the ground truth factors $W_t$ and $H_t$, where columns of $H_t$ belong to the unit simplex. 
Two crucial questions are: 
\begin{enumerate}

\item Under what conditions are the factors  $W_t$ and $H_t$ uniquely identifiable (up to trivial ambiguities such as permutation)? 
 
\item Does there exist a (polynomial-time) algorithm able to recover these ground truth factors $W_t$ and $H_t$? 

\end{enumerate}

Many works have studied these questions, leading to weaker and weaker conditions on the factors $W_t$ and/or $H_t$ that lead to uniqueness; see Section~\ref{relatedworks} for more details. 
Given that $W_t$ is identifiable, the identifiability of $H_t$ follows from well-known results:  $H_t$ is unique if and only if all columns of $X$ are located on $k$-dimensional faces of $\conv(W_t)$ having exactly $k+1$ vertices~\cite{sun2011underdetermined}. 
 When $W_t$ is full column rank, then $H_t$ is always unique as this condition is always met. 
 This is the reason why the identifiability results for SSMF are focused on the identification of $W_t$. In the remainder of this paper, we also only focus on the identifiability of $W_t$ in SSMF.



%
%

\paragraph{Contribution and outline of the paper} 


In this paper, we answer the two above questions in a novel way. 
First, in Section~\ref{relatedworks}, we review the main SSMF algorithms and identifiability results.  
Then, the main contributions of this paper are presented in the next four sections:  
\begin{enumerate}

 \item In Section~\ref{algoassumptions}, we provide a new identifiability conditions for SSMF, referred to as the facet-based conditions (FBC), that rely on the sparsity of $H$, by requiring to have $d = \rank(X)$ data points on each facet\footnote{A facet of a $d$-dimensional polytope is a $(d-1)$-dimensional face of that polytope. For example, in two dimensions, a polytope is a polygon and its facets are the segments.} 
 of $\conv(W)$; see Theorem~\ref{th:identifSSMFFBC}.  
As we will see, this condition is in most cases much weaker than the current state-of-the-art identifiability conditions that rely on the data points being sufficiently spread within $\conv(W)$.


\item In Section~\ref{algo}, we propose and study a first algorithm, dubbed brute-force facet-based identification (BFPI), for SSMF. 
BFPI looks for a polytope enclosing the data points by maximizing the number of points on each facet of that polytope. 
It relies on solving an optimization problem in the dual space. 	We provide an identifiability theorem for BFPI under the FBC (Theorem~\ref{th:identFPI}). 

\item In Section~\ref{algoseq}, we present a greedy variant for BFPI, namely GFPI, 
better suited for solving practical problems. GFPI extracts the facets of $\conv(W)$ containing the  largest number of data points sequentially by solving mixed integer programs (MIPs). 
We explain how GFPI is able to handle noise, rank deficient $W$'s, and outliers. We also provide an identifiability theorem for GFPI under the FBC (Theorem~\ref{th:indentGFPI}). 
	
	\item In Section~\ref{numexp}, we show on numerous numerical experiments that GFPI outperforms the current state-of-the-art SSMF algorithms. 
	In fact, GFPI allows us to recover the ground truth factor $W_t$ in much more difficult scenarios, while being less sensitive to noise and outliers. 	
	

\end{enumerate}

 




\section{Related Works: SSMF algorithms and identifiability}  \label{relatedworks}

Among the current approaches with identifiability guarantees for SSMF,  the two main ones are arguably separable NMF~\cite{arora2012computing, arora2016computing}, 
and simplex volume minimization~\cite{miao2007endmember}. 

 \paragraph{Separability}  
 
 
Separable NMF (SNMF) relies on the  separability assumption. 
It requires that each column of $W$ is present as a column of $X$, 
that is, that there exists an index set $\mathcal{K}$ such that $W = X(:,\mathcal{K})$. 
Equivalently, if separability holds, $H$ contains the identity as a submatrix. Separability is referred to as the pure-pixel assumption in 
HU~\cite{bioucas2012hyperspectral}, and to the anchor word assumption in topic modeling~\cite{arora2012computing}. 

  The separability assumption allows for efficient algorithms (that is, running in polynomial time) that are robust in the presence of noise; see~\cite{gillis2014and} and the references therein. 
  An instrumental algorithm to tackle separable NMF is the successive projection algorithm (SPA) introduced in~\cite{araujo2001successive}, and proved to be robust to noise in~\cite{gillis2013fast}. 
However, separability is a rather strong condition and might not hold in many applications.

 \paragraph{Minimum Volume, and Sufficiently Scattered Condition}  
    
To overcome this limitation, the Minimum-Volume (Min-Vol) framework was proposed which does not rely on the existence of the columns of $W$ in the data set. 
Min-Vol looks for a simplex that encloses the data points and simultaneously has the smallest possible volume. It can be formulated as follows~\cite{fu2015blind, lin2015identifiability}  
\begin{equation} \tag{Min-Vol} \label{eq:minvol}
\min_{W, H} \det(W^\top W) \quad 
\text{ such that } \quad X = WH \; \text{ and } \; H(:,j) \in \Delta^r \text{ for all } j. 
\end{equation} 
When the separability assumption is violated, 
Min-Vol is significantly superior to SNMF. 
Identifiability of Min-Vol requires $H$ to satisfy the sufficiently scattered condition (SSC), while $\rank(W) = r$. 
For a matrix $H \in \mathbb{R}_+^{r \times n}$ to satisfy the SSC, the columns of $H$ must be sufficiently scattered in $\Delta^r$ in order for their conical hull $\cone(H) = \{ y \ | \  y = Hx, x \geq 0 \}$ to  contain the second-order cone  \mbox{$\mathcal{C} = \{x \in \mathbb{R}^r_+ \big| e^\top x \geq \sqrt{r-1}||x||_2\}$}. 
	The SSC is a much more relaxed condition than separability, see Figure~\ref{overal} for an illustration. We refer the reader to~\cite{fu2015blind, fu2018identifiability, fu2018nonnegative} for more discussion on the SSC and the identifiability of SSMF. 
\begin{figure}[ht!]
	\begin{center}
		\includegraphics[width=\textwidth]{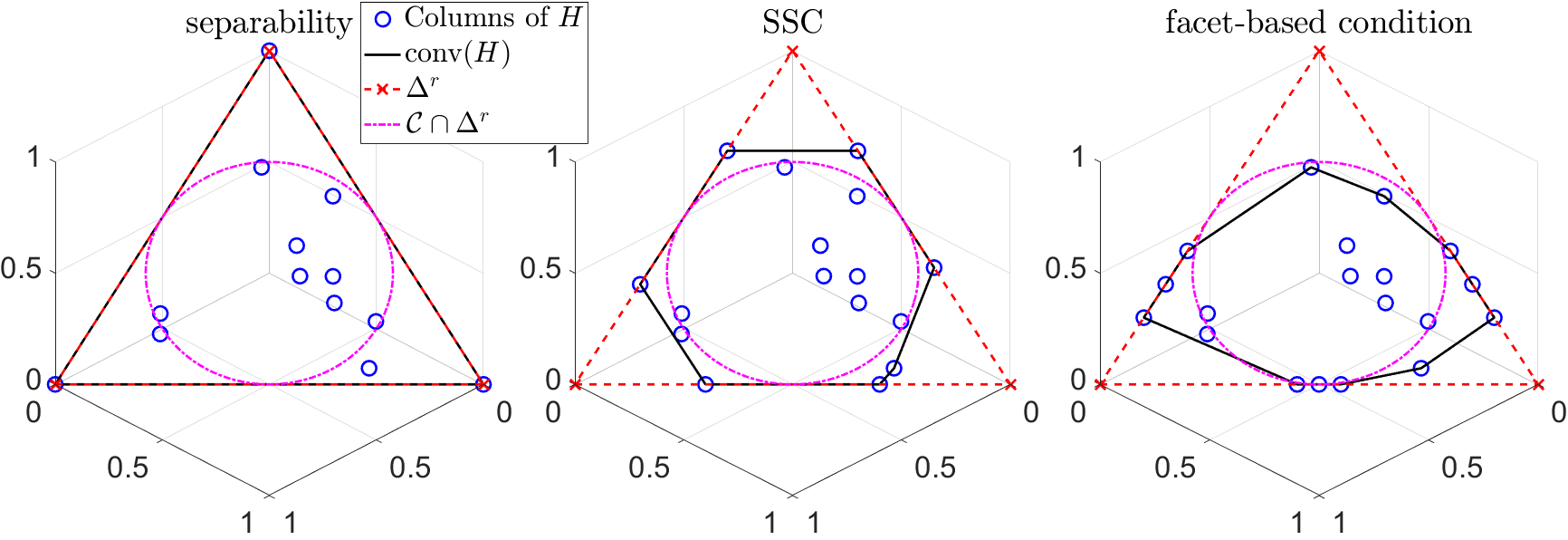}   
		\caption{Comparison of separability (left), SSC (middle), and our facet-based condition (right) for the matrix $H$ whose columns lie on the unit simplex.     
		On the left, separable NMF, as well as Min-Vol and FPI, will be able to uniquely identify $W$. 
	On the middle, separable NMF fails while 
	Min-Vol will uniquely identify $W$. Our approach may fail since the data points are also enclosed in another triangle containing six data points on its segments (there are only $r-1=2$ columns of $H$ on each facet of $\Delta^r$). 
			On the right, Min-Vol fails while FPI will be able to uniquely identify $W$. The reason Min-Vol fails is because the  triangle with minimum volume containing  the data points does not coincide with $\Delta^r$.  However, the only triangle with three data points on each segment and containing all data points is $\Delta^r$, which explains why FPI works. } 
		\label{overal}
	\end{center}
\end{figure} 

However, \ref{eq:minvol} is a difficult optimization problem and, as far as we know, most methods are based on standard non-linear optimization schemes (such as projected gradient methods) and come with no global optimality guarantees. Hence although  \ref{eq:minvol} allows for identifiability, it is still an open problem to provide an algorithm that solves the problem up to global optimality, in polynomial time; see the discussion in~\cite{fu2018nonnegative}. There exist non-ploynomial time algorithms for \ref{eq:minvol}; see the next paragraph. 



There are three main weaknesses for  \ref{eq:minvol}:  
\begin{enumerate}
	\item It requires $W$ to be full column rank. For example, in three dimensions, it can only identify three vertices. 
	
	\item It does not take advantage of the fact that, in many applications, most data points are usually located on the facets of the convex hull of the columns of $W$. In fact, in most applications, most columns of $H$ are sparse. 
	Minimum-volume NMF only uses the columns of $X$ that are not contained in the convex hull of the other columns, that is, it only uses the vertices of $\conv(X)$. 
	We believe this is a crucial information to take into account, and will lead to more robust approaches: we not only want to be able to reconstruct each data point, but also that as many points as possible are located on the facets of $\conv(W)$. 
	
	\item The SSC, although much milder than separability, is still a rather strong condition. It might not be satisfied in highly mixed scenarios; for example when a column of $W$ is not present in a sufficiently large proportion in sufficiently many pixels; see Figure~\ref{overal} (right) for an example.  
	
\end{enumerate}

In Section~\ref{algo}, we will provide a new weak condition for identifiability, namely the FBC. 
 In a nutshell, the FBC only requires to have $r$ data points on each facet of $\conv(W)$. (Note that the SSC implies that there are at least $r-1$ data points on each of these facets.) 
 Figure~\ref{overal} highlights the different identifiability conditions of the matrix $H$ in the case $r = 3$.  \\ 

Improving algorithmic designs for SNMF and Min-Vol is usually the main concern of the majority of recent studies; see for example~\cite{esser2012convex, 
recht2012factoring, 
arora2013practical, 
fu2015self,
gillis2014successive, 
leplat2019minimum, 
leplat2020blind}. In this paper, we take another direction, 
and consider new identfiability conditions, along with provably correct algorithms.

\paragraph{Algorithms based on facet identification} 

As mentioned before, 
our model and algorithm that will be presented in Section~\ref{algo} 
is based on the identification of the facets  of $\conv(W)$. 
There are few representative works that are based on similar ideas.  

Ge and Zou~\cite{GeZouNMF} introduced the concept of subset-separability which relaxes the separability condition. 
A factorization $X = WH$ is subset-separable if each column of $W$ is the unique intersection point of a subset of filled facets. A facet is filled if there is at least one point in the interior of the convex hull of the columns in $W$ corresponding to that facet or if the facet is exactly a vertex of $W$ (referred to as singleton set). 
 Based on this condition, they proposed the face-intersect algorithm to identify the filled facets of $\conv(W)$, and then their intersections corresponding to the columns of $W$.  
This algorithm is based on finding all facets by enumerating through all columns of $X$. 
The facets are identified using the following fact:  each point can be expressed as a convex combination of other points where the nonzero contributions correspond to points lying on the same facet. Hence, for each data point, the facet that it is lying on is identified, which leads to $n$ candidate facets. 
After eliminating the false positive facets (the ones which do not contain enough points) and close redundant facets, the intersection of the facets are determined as the basis matrix $W$. 
This algorithm requires the data points which are not in the lower dimensional faces to be in general positions, so that no random subset of points looks like a filled facet.  
The intuition behind our approach is related to these ideas. 
However our proposed algorithm will be completely different and our assumptions will be weaker: 
we do not require the facets to be filled, and do not put a general position condition on the points within the polytope $\conv(W)$. Moreover, as far as we know, the approach of Zou ang Ge is rather theoretical, and has not been used in real-world applications. 

Lin et al.~\cite{lin2015fast} proposed an algorithm that looks for the  simplex enclosing the data points by determining the $r$ associated facets, and then  calculating the vertices of that simplex (that is, the columns of $W$) by finding the intersection of the facets. 
Their approach is referred to as Hyperplane-based Craig-simplex-identiﬁcation (HyperCSI). 
In contrast to the previous approach, which produces many candidate facets, 
this approach generates exactly $r$ facets. The algorithm for identifying the facets relies on SPA~\cite{araujo2001successive}. 
First, SPA is utilized to estimate the $r$ \emph{purest} samples in $\conv(X)$ which are the points closest to the (unknown) vertices of $\conv(W)$ but not necessarily very close to facets. Let $\hat{W}$ denote the matrix whose columns contain these points. The $r$ facets are initially estimated as  $\mathcal{\hat{F}}_i =  \text{aff}( \{\hat{W}(:,1),\dots,\hat{W}(:,r)\}\backslash \{\hat{W}(:,i)\})$ for $i=1,…,r$ where $\text{aff}$
represents the affine hull. The orientational difference between the ground-truth facet and the estimated facet is reduced by finding \emph{active} samples that are close to the estimated facets. 
It was proven that in the noiseless setting, and as the number of columns of $X$ goes to infinity, that is, $n \rightarrow \infty$, the simplex identified by HyperCSI is exactly the minimum-volume simplex. The approach is simple and computationally efficient, however, as the purity decreases, the points selected by SPA might not be proper estimation of the purest samples and this makes the algorithm perform poorly; see the experimental results in Sections~\ref{noiseless_exp} and~\ref{noisy_exp}. 

In \cite{lin2017maximum}, Lin et al.\ proposed a different geometric approach for SSMF that is based on fitting a maximum-volume ellipsoid  inscribed in the convex hull of the data points. 
They refer to their algorithm as maximum volume inscribed ellipsoid (MVIE). 
They show that, under the SSC, the MVIE touches every facet of $\conv(W)$ which allows them to recover these facets, and then $W$. 
However, computing the MVIE requires to first compute all facets of $\conv(W)$, which is NP-hard in general (the number of facets can be exponential in the number of columns of $W$).  The second step uses semidefinite programming to compute the MVIE. 
As opposed to most algorithms for \ref{eq:minvol}, MVIE is guaranteed to recover $W$ in the noiseless case. However, the limitations of Min-Vol still hold here (see the discussion in the previous paragraph). 
Moreover, MVIE relies on facet enumeration algorithms that are sensitive to noise and outliers; see 
Section~\ref{numexp} for numerical experiments.  
This approach was recently improved by using a first-order method to solve the semidefinite program, and a different post-processing of the MVIE solution to recover $W$~\cite{lin2020nonnegative}. 


In~\cite{cohen2019identifiability}, authors provide identifiability results when the input matrix $H$ is sufficiently sparse. This result also applies to SSMF: it has a unique solution if on each subspace spanned by all but one column of $W$, there are $\floor{\frac{r(r-2)}{r-k}}+1$ data points with spark $r$ (that is, any subset of $r-1$ columns is linearly independent). 
However, this is a theoretical result, with no algorithm to tackle the problem. Moreover, this result does not take nonnegativity into account, and requires much more points on each facet than our facet-based condition. \\

In summary, as far as we know, algorithms for SSMF based on the identification of the facets of $\conv(W)$ have not been very successful in practice because they are either theoretically oriented, 
or they rely on strong conditions and are sensitive to noise.

\paragraph{Summary}


Table~\ref{related} summarizes the conditions under which SSMF algorithms  recover the ground truth factor $W$, in the noiseless case. 
It highlights five conditions: 
number of points per facet of $\conv(W)$ (this is essentially a sparsity condition on $H$),  
separability, 
SSC, 
full column rank of matrix $W$, and whether the number of samples needs to go to infinity.   
\begin{center}
	\begin{table*}[!ht]
		\begin{center}
			\caption{Indentifiability conditions for different SSMF algorithms in the exact case. SSMF is the model $X = WH$ where $W \in \mathbb{R}^{m \times r}$ and $H(:,j) \in \Delta^r$ for all $j$. We denote $d = \rank(X) \leq r$.} 
			\label{related}  
			\small\addtolength{\tabcolsep}{-1pt}
			\begin{tabular}{|c||ccccc|}
				\hline
				 & \thead{\begin{tabular}{@{}c@{}}\# points \\ per facets \end{tabular}} & \thead{separability} & \thead{SSC} & \thead{$d = r$} & $n \rightarrow \infty$ \\ \hline
				\begin{tabular}{@{}c@{}}Separable NMF \\ (SNMF)~\cite{araujo2001successive} \end{tabular} & $d-1$ & \checkmark & \checkmark & \checkmark & - \\ \hline
				\begin{tabular}{@{}c@{}}Simplex Volume Minimization\\ (Min-Vol)~\cite{miao2007endmember} \end{tabular} & $d-1$ & - & \checkmark & \checkmark & - \\ \hline
				\begin{tabular}{@{}c@{}}Maximize Volume Inscribed Ellipsoid\\ (MVIE)~\cite{lin2017maximum} \end{tabular} & $d-1$ & - & \checkmark & \checkmark & - \\ \hline
				\begin{tabular}{@{}c@{}}Hyperplane-based Craig-simplex-identiﬁcation\\ (HyperCSI)~\cite{lin2015fast} \end{tabular} & $d-1$ & - &\checkmark & \checkmark & \checkmark\\ \hline
				\begin{tabular}{@{}c@{}}Facet-based Polytope Identification\\ (BFPI and GFPI), this paper \end{tabular} & $d$ & - & - & - & -\\ \hline
			\end{tabular} 
		\end{center}
	\end{table*}
\end{center}

Our proposed algorithms, BFPI and GFPI, 
require $d = \rank(X)$ points per facet, which is only one additional data point on each facet compared to the other algorithms that require additional strong conditions such as the SSC or $\rank(W) = r$.  
Hence BFPI and GFPI will not always be stronger than Min-Vol 
(see Figure~\ref{overal} for an example), but they will be in most practical cases. 
Interestingly, their robustness (that is, their ability to perform well in the presence of noise) will depend on the fact that the data points are well spread on the facets. This is rather natural: in the presence of noise, it will be harder to identify a facet containing only points that are very close to one another.

\section{Identifiability of SSMF under the faced-based conditions (FBC)} \label{algoassumptions}

Let us state the FBC. 

\begin{assumption}[Facet-based conditions (FBC)] \label{ass1} 
Let $X \in \mathbb{R}^{m \times n}_+$ with $d = \rank(X)$, 
and let $W \in \mathbb{R}^{m \times r}$ and 
$H \in \mathbb{R}^{r \times n}_+$ be such that $X = WH$ where 
\begin{enumerate}
   
    \item[a.] No column of $W$ is contained in the convex hull of the other columns of $W$, that is, $\conv(W)$ is a polytope with $r$ vertices given by the columns of $W$. 
	
    
    \item[b.] The columns of $H$ belong to the unit simplex, that is, $H(:,j) \in \Delta^r$ for $j=1,2,\dots,n$.

    \item[c.] Each facet of $\conv(W)$ contains at least $s \geq d$ distinct columns of $X$ and, among them, at least $d-1$ generate that facet (that is, the dimension of the convex hull of these $s$ columns is $d-2$).  
    
    \item[d.] There are strictly less than $s$ distinct columns of $X$ on every facet of $\conv(X)$ which is not a facet of $\conv(W)$. 
    
\end{enumerate} 
\end{assumption}

Let us comment on these assumptions. 
\begin{itemize}

    \item Assumption~\ref{ass1}.a is necessary for any identifiable SSMF model since a column of $W$ cannot be identified if it is located in the convex hull of the other columns (it could be discarded to have a decomposition with $r-1$ factors).  
    
    Since $X = WH$,     we have $d = \rank(X) \leq \rank(W) \leq r$. However as opposed to most previous works, we do not assume $d=r$ so that $\conv(W)$ may contain more vertices than the ambient dimension plus one; for example, it could be a quadrilateral in the plane as in Figure~\ref{dualsquare}. 
    
    \item Assumption~\ref{ass1}.b allows for $WH$ to be a SSMF. 
    For NMF, that is, when $X = WH$ with $W \geq 0$ and $H \geq 0$, 
    Assumption~\ref{ass1}.b can be assumed without loss of generality by using a simple scaling of the columns of $X$ and $W$; see the introduction. 

    \item     The key assumption is Assumption~2.c. 
    It implies a certain degree of sparsity of the columns of $H$: a column of $X$ is on a facet of $\conv(W)$ if the corresponding column of $H$ has at least one zero entry. Hence Assumption~2.c implies that each row of $H$ has $d$ zero entries, and this condition is easy to check. 
    
    \item Assumption~\ref{ass1}.d will allow us to make the decomposition unique. For example, assume the data points are located on the boundary of a hexagon in two dimensions with $r=3$; see Figure~\ref{unique} for an illustration. There are many possible triangles that contain these points, and hence the factorization is not unique. Minimum-volume NMF picks the unique triangle with the smallest volume, while SSMF under the FBC picks the unique triangle having three points on each segment.   
    \begin{figure}[h!]
	\begin{center}
		\includegraphics[width=0.4\textwidth]{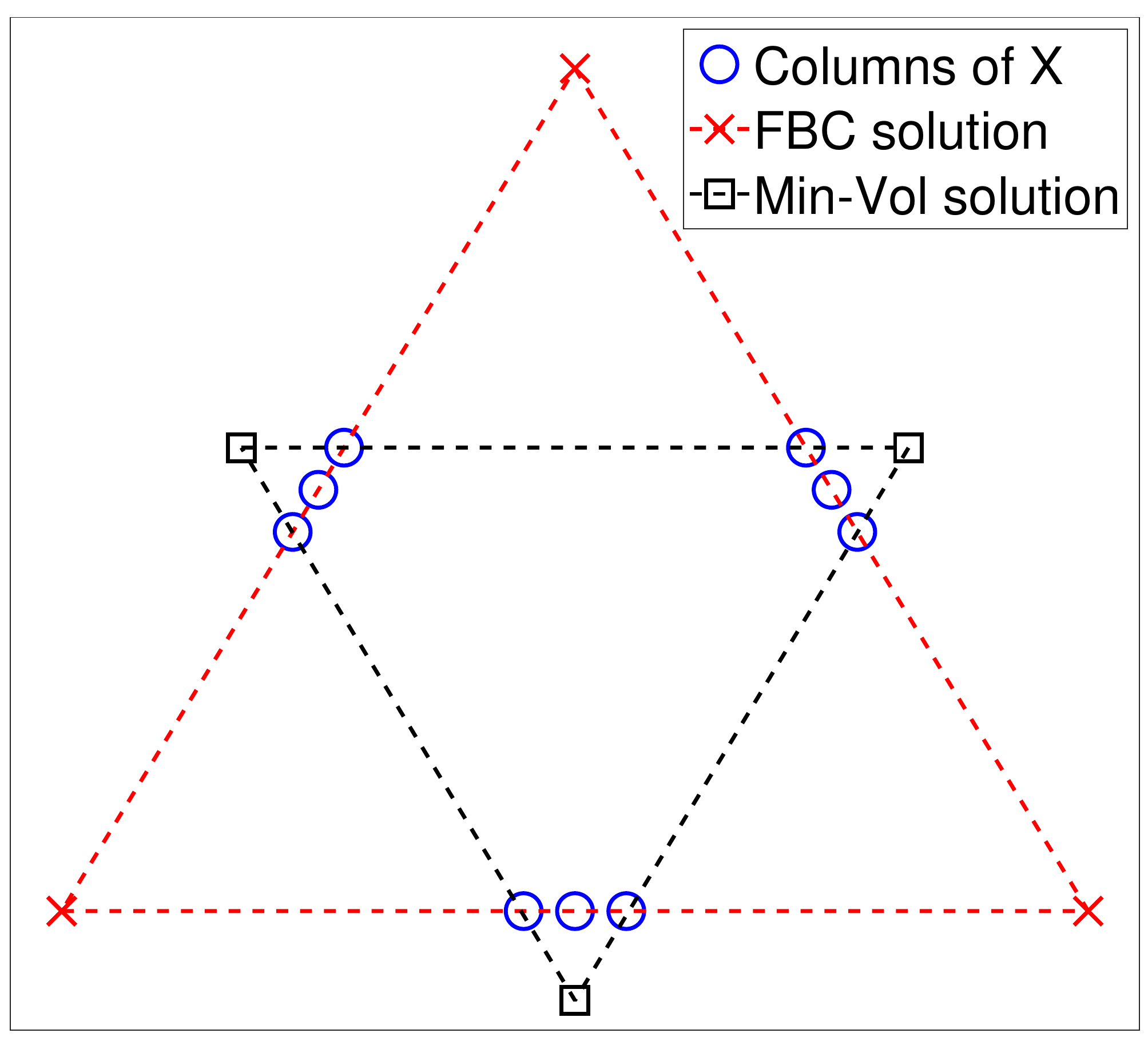}   
		\caption{Illustration of the non-uniqueness of SSMF. 
		SSMF under the FBC achieves uniqueness based on Assumption 1.d, and selects the triangle whose vertices are the red crosses, with three points on each segment. 
		Min-Vol selects the triangle whose vertices are the black squares, which has the smallest volume, but only two points on each segment.}  
		\label{unique}
	\end{center}
\end{figure}

   Under Assumption~\ref{ass1}.d, data points can be on the boundary of $\conv(X)$ as long as the number of such points on the same facet does not exceed the number of points on any of the facets of $\conv(W)$.
    We believe that this assumption will be met in most practical situations.  
    
    
    Assumption~\ref{ass1}.d is not easy to check as it requires to compute all facets of $\conv(X)$, and there could be exponentially many. 
    Note however that the SSC is NP-hard to check~\cite{huang2014non}. 
\end{itemize}

Compared to the assumption required for Min-Vol, our assumptions 
require one additional data points on each facet but does not require these data points to be well-spread on that facet. Moreover, we do not require $X$ to be of rank $r$. Note however that the well-spreadness of data points on a facet will influence the robustness to noise of our model; see Section~\ref{numexp}.

\begin{remark}[Separability vs.\ the FBC] 
 As opposed to the SSC, 
 Assumption~\ref{ass1} is not a generalization of separability because a separable matrix might not satisfy Assumption~\ref{ass1}.c. However, Assumption~\ref{ass1}.c could be relaxed as follows: either a facet of $\conv(W)$ satisfies Assumption~\ref{ass1}.c or its vertices are columns of $X$. In that case, our results still apply, using the same trick as in~\cite[Algorithm 5]{GeZouNMF}. 
We stick in this paper to Assumption~1.c for the simplicity of the presentation and because, in practice, it is not likely for a facet to contain all its vertices while not containing any point in its interior. 
For example,  in the hyperspectral unmixing application when $\rank(W) = r$, it would mean that
all endmembers but one are present as a column of $X$, 
and, except for the $d-1$ endmembers, all other pixels in the image contain some proportion of the last endmember not present as a column of $X$.  
Similarly, in the topic modeling application, this would mean that, except for the anchor words, all words are associated with the topic which does not have an anchor word in the data set. This is a very unlikely scenario in practice; also, this happens with probability zero under all reasonable probabilistic generative models we know of.  
\end{remark}

Before proving that the factor $W$ in SSMF is identifiable under the FBC (Assumption~\ref{ass1}), let us show the following lemma. 
\begin{lemma} \label{lem1}
 Let $X=WH$ satisfy Assumption~\ref{ass1}. Then every facet of $\conv(W)$ is a facet of $\conv(X)$.  
 \end{lemma}
 \begin{proof}
 Assumptions~\ref{ass1}.b implies $\conv(X) \subseteq \conv(W)$, while each facet of $\conv(W)$ contains at least $d$ columns of $X$ whose convex hull has dimension $d-2$ (Assumptions~\ref{ass1}.c). This implies that every facet of $\conv(W)$ is a facet of $\conv(X)$. 
 \end{proof} 

The proof of Lemma~\ref{lem1} leads to an interesting observation: for SSMF to be identifiable, one needs to have at least $d-1$ data points on each facet of $\conv(W)$, otherwise it cannot be a facet of $\conv(X)$ and hence cannot be identified. In fact, one can check that both separability and the SSC imply this condition. 
The FBC only requires one additional data point on each of these facets.

\begin{theorem}[Uniqueness of $W$ in SSMF under the FBC] \label{th:identifSSMFFBC}
Let $X = WH$ satisfying the FBC (Assumption~\ref{ass1}).  
For any other factorization $X = \hat{W}\hat{H}$ satisfying the FBC, $\hat{W} = W\Pi$ where $\Pi \in \{0,1\}^{r \times r}$ is a permutation matrix. 
\end{theorem}
\begin{proof} 
Note that the FBC depends on the parameter $s \geq d$. 
Assume there exists two factorizations $X = WH$ and $X = \hat{W}\hat{H}$ satisfying the FBC (Assumption~\ref{ass1}), where the parameter $s=s_W$ for $WH$, 
and $s=s_{\hat{W}}$ for $\hat{W}\hat{H}$. Assume without loss of generality that $s_W \leq s_{\hat{W}}$. 
By definition, the columns of $W$ and $\hat{W}$ are the intersections of the facets of $\conv(W)$ and $\conv(\hat{W})$, respectively. For $W$ and $\hat{W}$ to have at least one column that do not coincide (up to permutation), there is at least one facet of $\conv(W)$ that is different from one facet of $\conv(\hat{W})$. Let $\hat{\mathcal{F}}$ be a facet of $\conv(\hat{W})$ that is not a facet of $\conv(W)$. 
By Lemma~\ref{lem1}, $\hat{\mathcal{F}}$ is a facet of $\conv(X)$. 
This is in contradiction with Assumption~\ref{ass1}.d for $(W,H)$: $\hat{\mathcal{F}}$ is a facet of $\conv(X)$ but not a facet of $\conv(W)$ while it contains $s_{\hat{W}} \geq s_W$ distinct data points. 
\end{proof}


\section{Brute-force facet-based polytope identification (BFPI)} \label{algo}

In this section, we describe our first proposed algorithm, namely BFPI; see Algorithm~\ref{alg:bruteforce}. 
The high-level geometric insight of the proposed FPI algorithm is to identify the facets of $\conv(W)$, given the data points. 
Although we will not implement BFPI, we believe the high level ideas within BFPI are key, and may be an important starting point for future algorithmic design, which is the reason why we present it here. It was our starting point to develop GFPI presented in the next section. 

\algsetup{indent=2em}
\begin{algorithm}[ht!]
\caption{Brute-force facet-based polytope identification (BFPI) for SSNMF} \label{alg:bruteforce}
\begin{algorithmic}[1]
\REQUIRE Data matrix $X \in \mathbb{R}^{m \times n}$ satisfying Assumption~\ref{ass1}, and parameter $s$. 

    \ENSURE The basis matrix $W$. \vspace{0.2cm} 
    
     {\emph{\% Step 1. Preprocessing}} 
     \STATE Remove the zero columns of $X$, and remove duplicated data points. 
    
    \STATE Remove from each column of $X$ their average $\bar{x} = \frac{1}{n} \sum_{i=1}^n X(:,j)$ which lies in the interior of $\conv(X)$. We have 
    \[
    X = WH \quad \iff \quad X - \bar{X} = X - [\bar{x} \dots \bar{x}] = \left(W - [\bar{x} \dots \bar{x}] \right) H,  
    \]
    since the entries in each column of $H$ sum to one. Note that this reduces the rank of $X-\bar{X}$ to $d = \rank(X)-1$ 
    since 0 belongs to the convex hull of its columns. 
    
    \STATE Reduce the dimension of the columns of $X - \bar{X}$ to a $(d-1)$-dimensional space, by constructing the matrix $\tilde{X} \in \mathbb{R}^{(d-1) \times n}$ as follows.  
    Given the compact SVD of $X - \bar{X} = U\Sigma V^\top$ where $U \in \mathbb{R}^{m \times (d-1)}$, $\Sigma \in \mathbb{R}^{(d-1) \times (d-1)}$ and $V \in \mathbb{R}^{n \times (d-1)}$, $U$ and $V$ having orthogonal columns, we take  
    \[
    \tilde{X} \;  = \;  U^\top (X - \bar{X})  = \Sigma V^\top.
    \]
    Let us denote $\tilde{W} = U^\top \left(W - [\bar{x} \dots \bar{x}] \right)$, so that 
    $\tilde{X} = \tilde{W} H$. \vspace{0.2cm}  
    
     {\emph{\% Step 2. Compute all vertices of $\conv(X)^*$}} 
    
    \STATE Compute all vertices of $\conv(X)^* = \{ \theta \ | \ \tilde{X}^\top \theta \leq e \} \subseteq \mathbb{R}^{d-1}$. Let us denote these vertices $\{\theta_i\}_{i = 1}^v$.  \vspace{0.2cm}  
		
		 {\emph{\% Step 3. Identify the vertices of $\conv(W)^*$}} 
		
    \STATE Identify the vertices corresponding to a facet in the primal that contains more than $s$ data points
    \[
    J = \left\{ i \ \Big| \ \big| \{ j \ | \ \tilde{X}(:,j)^\top \theta_i = 1 \} \big| \geq s, 1 \leq i \leq v \right\}. 
    \]
     The convex hull of $\{\theta_i\}_{i \in J}$ is the dual of the convex hull of $\tilde{W}$. \vspace{0.2cm}   
     
     
      {\emph{\% Step 4. Recover $\tilde{W}$ from the vertices of $\conv(\tilde{W})^*$}} 
     
  \STATE Recover $\tilde{W}$ by intersecting the facets $\{ x \ | \ x^\top \theta_i \leq 1\}$ for $i \in J$, that is, compute the dual of  $\conv\left( \{\theta_i\}_{i \in J} \right)$. \vspace{0.2cm}  
  
   {\emph{\% Step 5. Postprocess $\tilde{W}$ to recover $W$}} 
  
  \STATE Project $\tilde{W} \in \mathbb{R}^{(d-1) \times r}$ back to the original $m$-dimensional space: 
  \[
  W =  U \tilde{W} + [\bar{x} \dots \bar{x}] . 
  \] 
\end{algorithmic}
\end{algorithm}

\paragraph{Preliminaries} Let $d = \rank(W)$. 
The facets of the $(d-1)$-dimensional polytope $\conv(W)$ are the polytopes of dimension $d-2$ obtained as the intersection of $\conv(W)$ with a hyperplane. 
For a set $\mathcal{A}$ containing the origin in its interior, 
we define its dual as 
\[
\mathcal{A}^* \;  = \; \left\{ y \ | \ x^\top y \leq 1 \text{ for all } x \in \mathcal{A} \right\}. 
\]
 If $\mathcal{A}$ is a polytope, then $\mathcal{A}^*$ is also a polytope whose facets correspond to the vertices of $\mathcal{A}$, and vice versa. 
Moreover, it is easy to prove that if 
$\mathcal{A} \subseteq \mathcal{B}$, then $\mathcal{B}^* \subseteq \mathcal{A}^*$. We refer the reader to~\cite{Ziegler95} for more information on polytopes. 
In order to recover the facets of $\conv(W)$, the dual space will be considered such that the problem of searching for the facets of a polytope is replaced by the equivalent problem of finding the vertices of a polytope in the dual space.

\paragraph{Preprocessing}  
Before doing so, 
the first step of FPI is to make sure the origin belongs to $\conv(W)$ by removing $\bar{x} = \frac{1}{n} \sum_{j=1}^n X(:,j)$ from all data points. This does not change the structure of the SSMF problem: 
\[
X(:,j) - \bar{x} = WH(:,j) - \bar{x} = (W - \bar{x} e^\top) H(:,j), 
\]
since $e^\top H(:,j) = 1$ because $H(:,j) \in \Delta^r$ for all $j$.  To simplify the notation, let us denote $\bar{X} = \bar{x} e^\top$. 
Then, to have a full-dimensional problem, that is,  to have the dimension of $\conv(X)$ coincide with the dimension of the ambient space, we project $X-\bar{X}$ onto its $(d-1)$-dimensional column space.  
In fact, since $0 \in \conv\big(X-\bar{X}\big)$, the rank of $X-\bar{X}$ is equal to $d-1$, and this second preprocessing step  amounts to premultiplying $X-\bar{X}$ by a $(d-1)$-by-$m$ matrix obtained via the truncated SVD of $X-\bar{X}$ (see Algorithm~\ref{alg:bruteforce} for the details).  
This does not change the structure of the SSMF problem either, it simply premultiplies $X$ and $W$ by a matrix of rank $d-1$.  
 This is a standard preprocesing step in the SSMF literature; see for example~\cite{ma2013signal}. 

\paragraph{Dual approach}   

Let us denote the dual of $\conv(X)$ as  
\[
\conv(X)^* 
\; = \; \left\{ \theta \ | \ x^\top \theta \leq 1 \text{ for all } x \in \conv(X)  \right\} 
\; =\;  \left\{ \theta \ | \ X^\top \theta \leq e  \right\} . 
\] 
Since $\conv(X) \subseteq \conv(W)$, the dual of $\conv(W)$ is contained in $\conv(X)^*$. Let us illustrate this on a simple example.
\begin{example}
Let the columns of $W$ be the vertices of the square $[-1,1] \times [-1,-1]$, while 
\[
X = \left( \begin{array}{cccccccccccc}
-1 & -1  & -1  & -0.8 & -0.65  & -0.5 & -0.8 & -0.65 & -0.5 & 1 &  1  & 1 \\
0.8 & 0.65 & 0.5 &  1  & 1   &  1  &  -1 &  -1 &  -1 &  -0.8 & -0.65 & -0.5
\end{array} \right), 
\]
see Figure~\ref{dualsquare} for an illustration. 
The polygon $\conv(X)$ has 8 segments: 4 containing 3 data points, and 4  containing 2 data points.  
In the dual space, 4 of the vertices of $\conv(X)^*$ correspond to the 4  vertices of $\conv(W)^*$, that is, to the four segments of $\conv(W)$, while the other 4 correspond to the other 4 segments of $\conv(X)$. 
\begin{figure}[ht!]
\begin{center}
\includegraphics[width=0.9\textwidth]{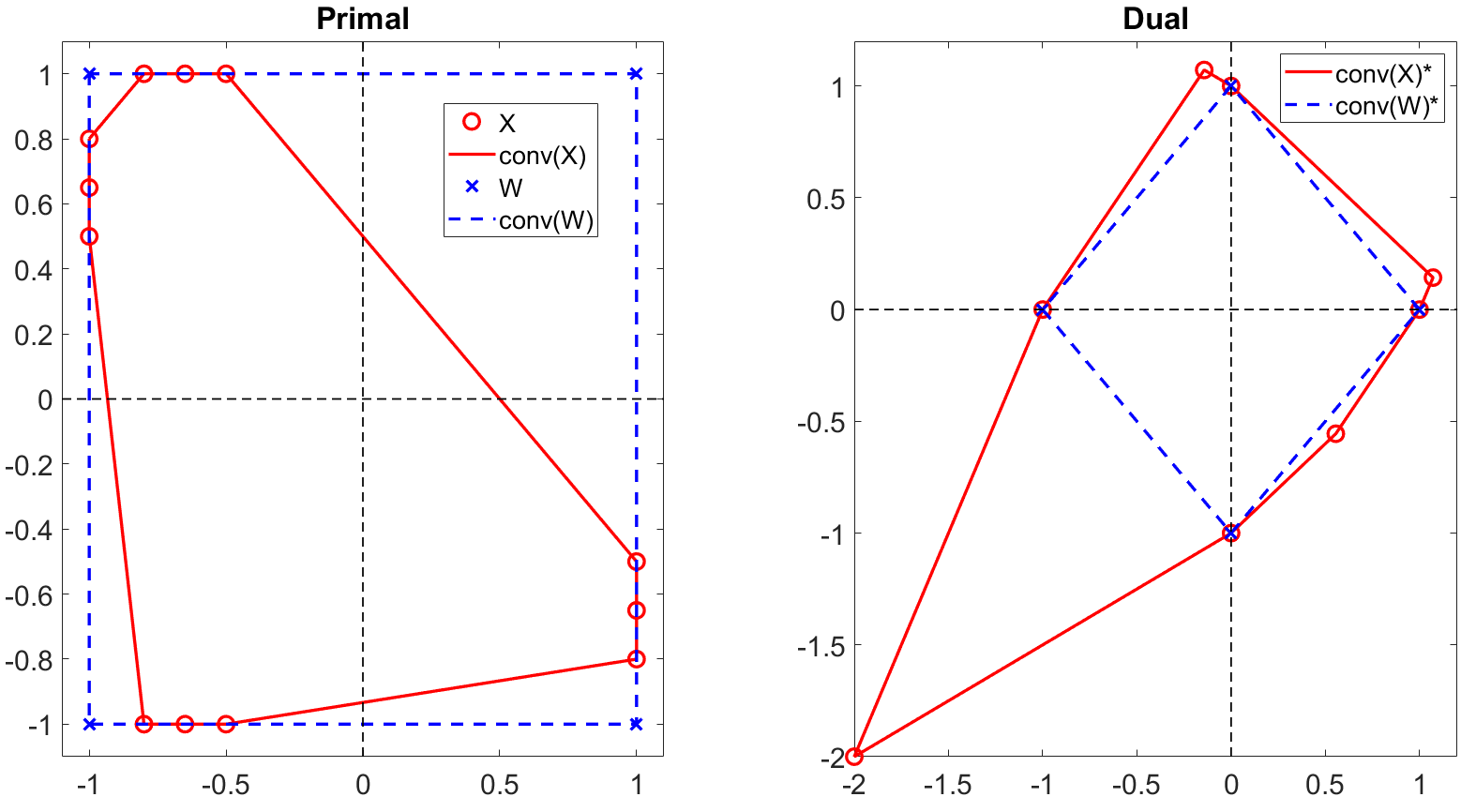}   
\caption{Illustration of the concept of duality to compute SSMF. 
On the left, this is the primal space where $\conv(X) \subseteq \conv(W)$.  
On the right, the circles are the vertices of $\conv(X)^*$ in the dual space corresponding to the segments of $\conv(X)$ in the primal. 
The crosses  are the vertices of $\conv(W)^*$ corresponding to the segments of $\conv(W)$ in the primal. This is the dual representation where $\conv(W)^* \subseteq \conv(X)^*$.  
\label{dualsquare} 
} 
\end{center}
\end{figure}
\end{example}


Our goal is to find the vertices of $\conv(X)^*$ that correspond to the vertices of $\conv(W)^*$, that is, the facets of $\conv(W)$. 
Under Assumption~\ref{ass1}.c, there are at least $d$ columns of $X$ on each facet of $\conv(W)$ whose convex hull has dimension $d-2$  (on Figure~\ref{dualsquare}, there are three points on each segment of $\conv(W)$). This implies that a subset of the vertices of $\conv(X)^*$ contains the vertices of $\conv(W)^*$, as shown in the following lemma. 
 \begin{lemma} \label{lem1bis}
 Let $X=WH$ satisfy Assumption~\ref{ass1}, and assume $X$ has been preprocessed as described in Algorithm~\ref{alg:bruteforce} so that $0 \in \conv(X)$ and  $X \in \mathbb{R}^{(d-1) \times n}$ where 
 $\rank(X) = d-1$. 
 Then the set of vertices of $\conv(X)^*$ contain all the vertices of $\conv(W)^*$. 
 \end{lemma}
 \begin{proof}
 This follows from Lemma~\ref{lem1} and duality.  
 \end{proof} 
Once the vertices of $\conv(X)^*$ are identified, we recover the vertices of $\conv(W)^*$ that correspond to the facets of $\conv(W)$ containing the largest number of data points. More precisely, under Assumption~\ref{ass1}, we have the following lemma. 
 \begin{lemma} \label{lem2} Let $X=WH$ satisfy Assumption~\ref{ass1}, and assume $X$ has been preprocessed as described in Algorithm~\ref{alg:bruteforce} so that $0 \in \conv(X)$, $X \in \mathbb{R}^{(d-1) \times n}$ where $\rank(X)=d-1$, and $X$ does not have duplicated columns.   
 Then  
 the set $\{ x \in \conv(W) \ | \ \theta^\top x = 1\}$ is a facet of  $\conv(W)$ if and only if  
 \begin{equation} \label{facet}
\theta \text{ is a vertex of } \conv(X)^* = \big\{ \theta \ | \ X^\top \theta \leq e \big\}  
\quad \text{ and } \quad 
\big| \{ j \ | \ X(:,j)^\top \theta = 1 \} \big| \geq s ,  
 \end{equation}
 where $|\mathcal{A}|$ denotes the cardinality of the set $\mathcal{A}$. 
 \end{lemma}
 \begin{proof}
 Let $\{ x \in \conv(W) \ | \ \theta^\top x = 1\}$ be a facet of $\conv(W)$. By Lemma~\ref{lem1}, 
$\theta$ must belong to $\conv(X)^*$, while, by Assumption~\ref{ass1}.c, facets of $\conv(W)$ contain more than $s \geq d$ columns of $X$. 

Let $\theta$ satisfy \eqref{facet} so that the set $\mathcal{F} = \{ x \in \conv(W) \ | \ \theta^\top x = 1\}$ contains $s$ columns of $X$. 
Since $\theta$ is a vertex of $\conv(X)^*$, the set 
$\mathcal{F}$ corresponds, by duality, to a facet of $\conv(X)$. 
By Assumption~1.c, the facets containing at least $s$ points must correspond to facets of $\conv(W)$.  
 \end{proof} 
 
Finally, $W$ is recovered by intersecting the facets of $\conv(X)$ containing more than $s$ data points.  
The proposed brute-force algorithm is presented in Algorithm~\ref{alg:bruteforce}.
The main step of Algorithm~\ref{alg:bruteforce} is a {vertex enumeration} problem in the {dual space}. 




\paragraph{Identifiability} 

Let us prove that, if $X=WH$ satisfies Assumption~\ref{ass1}, then Algorithm~\ref{alg:bruteforce} recovers $W$, up to permutation of its columns.  

\begin{theorem}[Recovery of $W$ by Algorithm~\ref{alg:bruteforce}] \label{th:identFPI}
Let $X=WH$ satisfy Assumption~\ref{ass1}. Then Algorithm~\ref{alg:bruteforce} recovers the columns of $W$ (up to permutation).  
\end{theorem}
\begin{proof}
First, as already noted above, the prepossessing step does not change the geometry of the problem, that is, if $X = WH$ satisfies Assumption~\ref{ass1}, then $\tilde{X} = \tilde{W}H$ also satisfies Assumption~\ref{ass1}. Hence let us assume w.l.o.g.\ that $0 \in \conv(X)$ and $X \in \mathbb{R}^{(d-1) \times n}$ where $\rank(X) = d-1$. 

The rest of the proof follows from Lemmas~\ref{lem1} and~\ref{lem2}. 
By Lemma~\ref{lem1}, the vertices of $\conv(X)^*$ computed in step 4 of Algorithm~\ref{alg:bruteforce} correspond to facets of $\conv(X)$. 
 By Lemma~\ref{lem2}, only the facets of $\conv(X)$ corresponding to facets of $\conv(W)$ contain at least $s$ columns of $X$. 
\end{proof}

\paragraph{Computational cost} Algorithm~\ref{alg:bruteforce} may run in the worst-case in exponential time. The set $\conv(X)^*$ is an $(d-1)$-dimensional polytope defined by $n$ inequalities and can have exponentially many vertices, namely $O\left( \binom{n}{d-1} \right)$. 
In the next section, we propose a greedy algorithm that identifies facets sequentially. Moreover, we will adapt it so that it can handle noise and outliers. \\

Although we could adapt BFPI to handle noisy input matrices, we will develop in the next section a more practical algorithm that sequentially extracts the facets of $\conv(W)$, and hence will not require to identify all vertices of $\conv(X)^*$. However, we believe BFPI is important, and could be the starting point for other practical SSMF algorithms based on facet identification.


\section{Greedy FPI (GFPI)} \label{algoseq} 

The brute-force approach presented in the previous section is provably correct but may require exponentially many operations.  
Note that the same observation holds for Min-Vol: as far as we know, the algorithms that provably solve Min-Vol up to global optimally require to compute all facets of $\conv(X)$; see Section~\ref{relatedworks}. 
In this section, we propose a more practical sequential algorithm, 
dubbed Greedy FPI (GFPI), 
by leveraging highly efficient MIP solvers (in particular their ability to quickly find high quality feasible solutions).   
Although it is still computationally heavy to solve (that is, we cannot prove it runs in polynomial time), it is much more practical than BFPI and allows to solve large problems; see Section~\ref{numexp}. \\  

GFPI sequentially searches for the facets of $\conv(X)$ containing the largest number of points that lie on them. 
This section is organized as follows. 
The optimization model used to identify such a facet, even in the presence of noise and outliers, is described 
in Section~\ref{sec:facetid}. 
Once a facet is identified, the same model can be used to extract the next facet, but this requires to remove the previously identified facets from the search space (Section~\ref{sec:cutfacets}). 
To make sure the intersection of the $r$ extracted facets corresponds to a bounded polytope, we need to add a constraint when extracting the last facet (Section~\ref{sec:bounded}). 
The way the matrix $W$ is estimated from the extracted facets is described in Section~\ref{sec:intersect}. 
Finally, in Section~\ref{discuss}, we prove the identifiability of GFPI under the FBC, and discuss its computational cost and the choice of its parameters.  
	

\subsection{Identifying a facet, in the presence of noise and outliers} \label{sec:facetid}

%
%

As for GFPI, the data points are first centered and projected into a $(d-1)$-dimensional subspace 
to obtain $\tilde{X} \in \mathbb{R}^{(d-1) \times n}$ such that $0 \in \conv\big(\tilde{X}\big)$ and $\rank\big(\tilde{X}\big) = d-1$.  
Since we want GFPI to handle noisy data, we cannot use the metric of the number of points on a facet of $\conv\big(\tilde{X}\big)$ to know whether it is also a facet of $\conv\big(\tilde{W}\big)$, because points will not be exactly located on the facets of $\conv(X)$.   
Given a parameter $\gamma$ that depends on the noise level, we propose to solve 
\begin{equation} \label{indicator_first} 
\max_{\theta \in \mathbb{R}^{d-1}} \; \sum_{j=1}^{n}  \ \mathbf{I}\left(\tilde{X}(:,j)^\top\theta \geq 1-\gamma\right) \quad 
\text{ such that }  \quad 
\tilde{X}^\top \theta \leq (1+\gamma)e ,  
\end{equation} 
where $\mathbf{I}(.)$ is the indicator function which is equal to 1 if the input condition is met, and to 0 otherwise. 
The variable $\theta$ encodes the facet $\{ x \in \conv\big(\tilde{X}\big) \ | \ x^\top \theta = 1 \}$. 
The optimal solution of~\eqref{indicator_first} corresponds to a facet containing the largest number of data points within a safety gap defined by $\gamma$. 
In the noiseless case, taking $\gamma = 0$ and solving~\eqref{indicator_first} provides  a facet of $\conv(X)$ containing the largest number of columns of $X$, 
and hence it will correspond to a facet of $\conv(W)$, under Assumption~\ref{ass1}; see Lemma~\ref{lem2}. 

To solve~\eqref{indicator_first},  
we use a MIP. 
We introduce a binary variable $y_i \in \{0,1\}$ ($1 \leq i \leq n$) which is equal to 0 if $\mathbf{I}(\tilde{X}(:,i)^\top\theta \geq 1-\gamma) = 1$, and to 1 otherwise\footnote{We made this (arbitrary) choice to obtain a minimization problem, which is more standard.}, 
and solve 
\begin{equation*}  
\min_{\theta \in \mathbb{R}^{d-1}, y \in \{0,1\}^n}  \sum_{j=1}^{n} \ y_j   \quad 
\text{ such that } 1-\gamma-Ay_j \leq \tilde{X}(:,j)^\top\theta \leq 1 + \gamma \text{ for } 1 \leq j \leq n.  \nonumber 
\end{equation*}
The parameter $A$ is a sufficiently large scalar based on the {BIG-M} approach often used to cast optimization problems with indicator functions; see Remark~\ref{rem:A}. 
If the condition $\tilde{X}(:,j)^\top\theta \geq 1-\gamma$ is satisfied, the value of $y_j$ can be either 0 or 1. Since the MIP minimizes $y_j$, $y_j$ will be set to 0. If it is not satisfied, that is, $\tilde{X}(:,j)^\top\theta < 1-\gamma$, then the value of $y$ has to be equal to 1. Note that $y_j = 0$ means that the corresponding data point is located close to the sought facet. 

We have observed numerically that using the same safety gap for the $n$ constraints  $\tilde{X}^\top \theta \leq (1+\gamma)e$ does not give enough degrees of freedom to the formulation, and, in difficult scenarios, fails to return good solutions. In particular, it is unable to deal with outliers that might be arbitrarily far away from the sought polytope of which $\{ x \ | \ \theta^\top x \leq 1\}$ is a facet. 
Hence we introducing the variable $\delta \in \mathbb{R}^n_+$ that accounts for the distance of the data points  from the polytope; in particular, $\delta_j = 0$ if $\theta^\top \tilde{X}(:,j) \leq 1$. 
We propose the following MIP  
\begin{align} 
\min_{\theta,\delta \geq 0, y \in \{0,1\}^n} 
& \quad \sum_{j=1}^{n} \ y_j +\lambda \sum_{j=1}^{n} \delta_j   \nonumber   \\
\text{ such that } & \  
 1-\gamma-Ay_j \leq \tilde{X}(:,j)^\top\theta \leq 1 + \delta_j   
 \text{ for } 1 \leq j \leq n, \label{noisy_bigm} \\ 
  & \ \delta_j \leq Ay_i + \gamma \text{ for } 1 \leq j \leq n. \nonumber     
\end{align}  
The parameter $\lambda$ controls how much the points are allowed to be far away from the polytope. 
The constraint $\delta_j \leq Ay_i + \gamma$ forces the binary variable $y_j$ to get the value of 0 only when $\big|\tilde{X}(:,j)^\top\theta - 1 \big| \leq \gamma$, so that the data point is in fact close to the facet, up to the safety gap $\gamma$. 
The entries of $\delta$ larger than $\gamma$ will correspond to outliers, that is, points that are outside and far away from the sought polytope.


\begin{remark}[Value of $A$] \label{rem:A}
The BIG-M formulation is frequently used as a modeling trick for problems with disjunctive or indicator constraints; see for example~\cite{belotti2016handling} and the references therein. 
	The scalar $A$ is a parameter, and a good choice for its value depends on the data. A very large value for $A$ would lead to weak relaxations while very small values might lead to cutting off feasible solutions. Choosing a good value for $A$ is a difficult problem in MIP literature \cite{bonami2015mathematical}. 
	We have set $A$ to 10 in all the experiments in the absence of outliers and did not notice sensitivity to this value. For the experiments with outliers, we used $A = 100$; this makes sense as outliers are further away from $\conv(W)$. 
\end{remark}

\subsection{Cutting previous facets from the solution space} \label{sec:cutfacets}

Solving~\eqref{noisy_bigm} allows to approximate one  facet of $\conv\big(\tilde{W}\big)$. In order to extract other facets sequentially, 
we need to eliminate the previously found facets from the feasible solutions of~\eqref{noisy_bigm}. 
To do so, we select one point in each of the previously identified facets such that it \emph{only} belongs to the corresponding facet, that is, it needs to be in the relative interior of that facet. This point is chosen as the average of the data points associated to that facet. 
We will denote $M^{(t)} \in \mathbb{R}^{(d-1) \times t}$ the matrix whose columns correspond to these points after $t$ facets have been identified. 
At the next step, that is, at the $(t+1)$th step, 
we restrict the search space of~\eqref{noisy_bigm} by adding the following constraints making sure that these selected points do not lie on the current sought facet: 
\begin{equation} \label{cutpreviousfacets}
\theta^\top{M^{(t)}(:,i)} \ \leq 1 - \gamma - \eta \quad \text{for} \ i=1,\dots,t,
\end{equation}
where $\eta \in \mathbb{R}_+$ is a margin parameter which controls how far the next facet should be from the previously selected facets. The larger $\eta$ is, the further the facets will be from each other. 
Figure~\ref{margin} illustrates this procedure after one facet has been identified (corresponding to $\theta_1$ on the figure), in the primal and dual spaces simultaneously. As the margin parameter $\eta$ increases, more and more feasible solutions are cut from the dual $\conv(X)^*$. However, for all margin values, namely $\{0.1, 0.5, 0.8\}$, the two other vertices of $\conv(W)^*$ are not cut. 
\begin{figure}[ht!]
	\begin{center}
\includegraphics[width=0.6\textwidth]{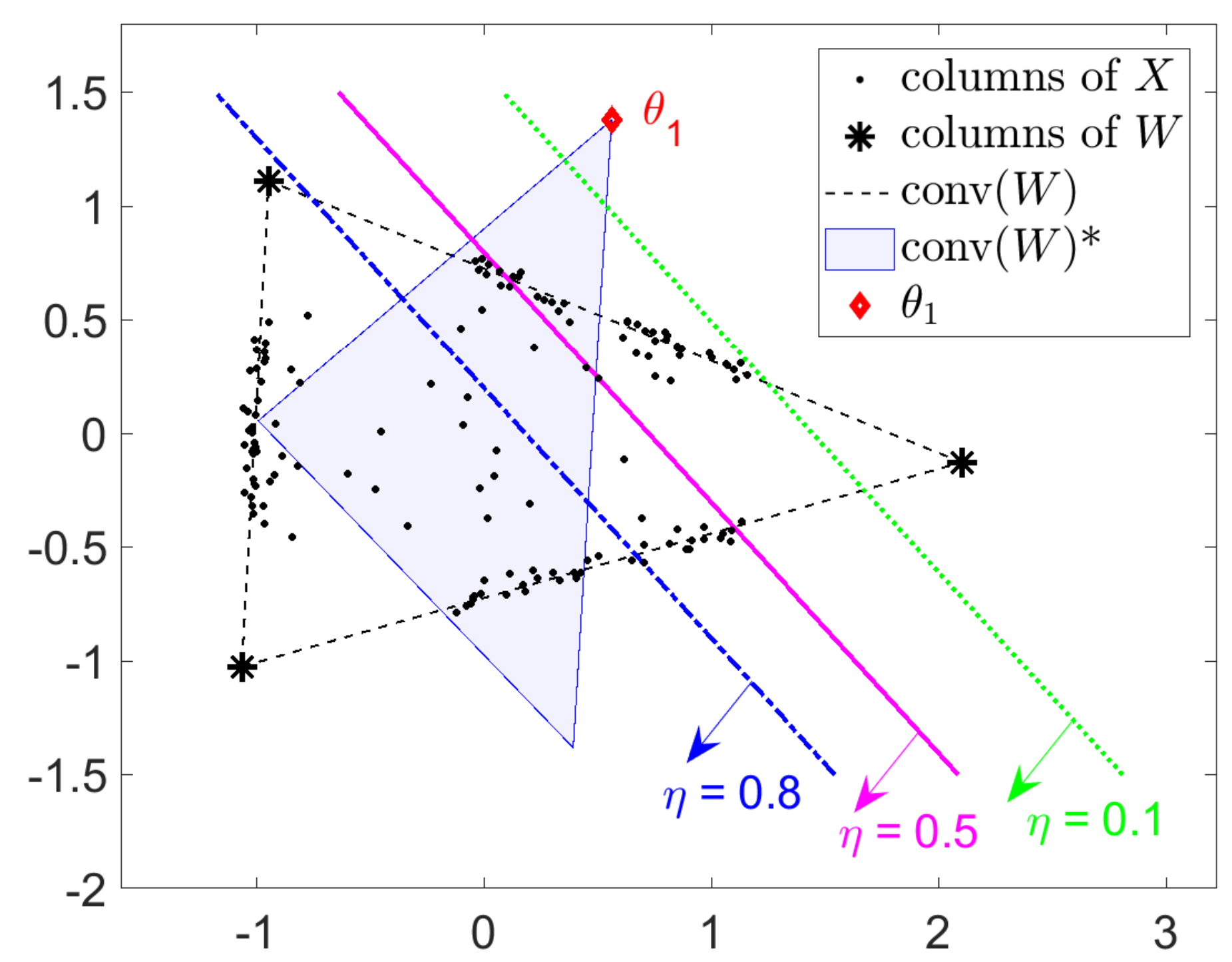}   
		\caption{Illustration of the effect of the margin parameter $\eta$ on the solution space with $r=3$.  
		\label{margin}
		} 
	\end{center}
\end{figure}
In general, if the margin value $\eta$ is set too high, there will be no feasible solution to the optimization problem and, if it is set too low, the algorithm might find a facet too close to the previously identified facets. 
	However, both cases can be prevented. If the optimizing algorithm does not find any feasible solution, the margin can be reduced. If the identified facet is not sufficiently different from the other ones, it can be increased. However, as shown in Section~\ref{app:param}, our approach is not too sensitive to this parameter.  

	\begin{remark}[Construction of $M^{(t)}$] 
	 If the data points associated to a facet are not well-spread in that facet, 
	 their average  might lie near the boundary of that facet. (Note however that, by Assumption~\ref{ass1}, the points on a facet generate that facet hence their average has to be in the relative interior of the facet.) 
	In this situation, a separable NMF algorithm, such as SPA, can be used to identify $d-1$ points well spread on this facet, and then take the average of this subset of points. 
	In this paper, we use the successive nonnegative projection algorithm (SNPA)~\cite{gillis2014successive} which is more robust to noise than SPA. 	This second strategy is useful in more difficult scenarios, and we have used it for the real-world hyperspectal images in Section~\ref{sec:hsi}. 
\end{remark}



\subsection{Obtaining a polytope} \label{sec:bounded} 

We are now able to extract sequentially facets of $\conv(X)$ that approximately contain the largest number of columns of $X$. Let us focus on the case $W$ is full column rank, that is, $\rank(W) = r$. In difficult scenarios, for example when $W$ is ill-conditioned, or the noise level is high, 
we cannot guarantee that, after having extracted $r$ facets, we will obtain a polytope (that is, a bounded polyhedron). 
In order to resolve this issue, we take advantage of the following theorems. 
 \begin{theorem}[Boundedness theorem \cite{salmani2018rotating}]
 	Let $\theta_1,\dots,\theta_d$ be $d$ linearly independent vectors in $\mathbb{R}^d$. If $\theta^{d+1} = -\sum_{i=1}^{d}\mu_i \theta^i$ with $\mu_i >0$ for $\forall i \in \{1,\dots,d\}$, then the positive hull of these $d+1$ vectors span $\mathbb{R}^d$. 
 \end{theorem}
\begin{theorem}[Full body theorem \cite{salmani2018rotating}]
	Given a set $\Theta = \{\theta_1,\dots,\theta_\ell\}$ in $\mathbb{R}^d$ and a polyhedron $\mathcal{P} = \{x | \theta_i^\top x \leq b_i; i=1,\dots,\ell\}$, the polyhedron is bounded if and only if the positive hull of the set $\Theta$ spans $\mathbb{R}^d$.
\end{theorem}
To ensure that the $r$ identified facets define a bounded polytope in $\mathbb{R}^{d-1}$, we add the following constraint to~\eqref{noisy_bigm} when computing the last facet:
\begin{equation} \label{bounded}
\theta = -\sum_{i=1}^{d-1} \mu_i \theta^{(i)} 
 \quad \text{with} \quad  \mu_i \geq \epsilon \text{ for } i=1,\dots,d-1,  
\end{equation}
where $\theta^{(i)}$ ($1 \leq i \leq r-1$) are the $r-1$ vectors extracted at the first $r-1$ steps of GPFI, 
and 
 $\epsilon$ is a small positive constant. We used $\epsilon = 0.1$ in all numerical experiments in Section~\ref{numexp}. 
 
 As mentioned above, this additional constraint plays an instrumental role in difficult scenarios. 
 For example, on the real hyperspectral images from  Section~\ref{sec:hsi} that are highly contaminated with noise (and do not follow closely the model assumptions), this constraint allowed us to obtained significantly better solutions; see in particular Figure~\ref{samson_W}-(b) where one of the extracted facet does not have  many points around it: 
 its extraction was made possible because of~\eqref{bounded}. 
 Moreover, we have observed that the use of~\eqref{bounded} makes the identification of the last facet less sensitive to the  margin parameter $\eta$ as~\eqref{bounded} forces the sought facet to be far from the facets already identified.


\paragraph{Rank-deficient case}


An advantage of our proposed sequential approach is that it is not required that $\rank(W) = r$. In fact, our sequential strategy can be used to extract more than $r$ facets of $\conv(W)$ when $\rank(W) < r$; for example, in Section~\ref{sec:rankdef}, we will extract the 4 segments of a square. 
In practice, in the rank-deficient case, it is unclear how many facets need to be extracted: this depends on how many columns of $W$ need to be identified. 
In two dimensions, the number of facets of a polygon coincides with the number of vertices. However, in higher dimensions, the number of facets and vertices cannot be deduced from one another. Hence we leave to the user to decide how many facets need to be extracted. A possible heuristic would be to extract facets as long as they contain sufficiently many data points, and/or as long as the corresponding polyhedron is unbounded. We leave this as a direction of further development.

\subsection{Summary of the MIP model for facet identification}  

To summarize, GFPI will extract one facet at each iteration. At iteration $t$, it solves the following MIP:  
 \begin{align} \label{mixedinteger_fr_outlier}
 \min_{\theta \in \mathbb{R}^{d-1},\delta \in \mathbb{R}^{n}_+,y \in \{0,1\}^n} & \; \sum_{j=1}^{n} \ y_i +\lambda \sum_{j=1}^{n} \delta_j & \\
\text{ such that } 
& \tilde{X}(:,j)^\top \theta \leq 1 + \delta_j \text{ for } 1 \leq j \leq n, 
& \rightarrow \text{Forming dual space} \nonumber \\
& \tilde{X}(:,j)^\top \theta \geq 1 - \gamma - Ay_j \text{ for } 1 \leq j \leq n, 
& \rightarrow \text{Counting points on the facet} \nonumber \\
& \theta^\top M^{(t-1)}(:,k) \leq 1-\gamma - \eta \text{ for } 1 \leq k \leq t-1, 
& \rightarrow \text{Removing previous facets} 
\nonumber \\
& \delta_j \leq Ay_j + \gamma \text{ for } 1 \leq j \leq n. & \rightarrow \text{Discarding outliers}  \nonumber 
\end{align}
The optimal solution of \eqref{mixedinteger_fr_outlier} at iteration $t$ for the variable $\theta$ will be denoted $\theta^{(t)}$, it approximates the $t$th facet of $\conv\big( \tilde{W} \big)$. 
When $\rank(W) = r$, the constraint~\eqref{bounded} is added when extracting the last facet in order to make sure the polytope defined by the extracted facets is bounded; see Section~\ref{sec:bounded}.   

The proposed MIP model~\eqref{mixedinteger_fr_outlier} has been carefully designed in order to achieve state-of-the-art performances on synthetic and real-world data sets; see Section~\ref{numexp} for the numerical experiments. 
It results from a long trial-and-error procedure, and many alternative formulations have been tested. 
A direction of research is to further improve this MIP formulation. 

\subsection{Post-processing: intersection of facets} \label{sec:intersect}

Once the facets of $\conv\big( \tilde{W} \big)$ are identified, that is, the vectors $\{ \theta^{(t)} \}_{t=1}^T$  are computed sequentially using~\eqref{mixedinteger_fr_outlier}, how can we recover $\tilde{W}$ accurately, even in noisy conditions? 
We have observed that it is possible to improve the quality of the identified facets, and hence of $\tilde{W}$, by taking advantage of the knowledge of the data points associated to them. 

For the identified facet corresponding to the vector $\theta^{(t)}$ ($1 \leq t \leq T$), 
let 
\[ 
		J^{(t)} = \left\{ j \ \big| \ \left| \tilde{X}(:,j)^\top \theta^{(t)} - 1 \right| \leq \gamma \right\} 
		\]
	be the index set containing the points associated to it. The set $J^{(t)}$ contains the indices such that $y_j = 0$ when solving~\eqref{mixedinteger_fr_outlier}. 
To improve the estimate of  $\theta^{(t)}$, 
we compute the normal vector of the affine hull containing the columns of $X(:,J^{(t)})$, which is  the left singular vector corresponding to the smallest singular value of the SVD of $X(:,J^{(t)})$, after removing the average from each column (the facet is translated so that 0 belongs to it).  
	 Let us denote $\Theta \in \mathbb{R}^{d-1 \times T}$ the matrix whose columns are these singular vectors so that $\Theta(:,t)$ replaces $\theta^{(t)}$. 
The facet $t$ has the form  $\{ x \ | \ \Theta(:,t)^\top x = q_t \}$ for some offset $q_t$. Again, we compute $q_t$ from the data by taking the average dot product between the normal vector $\Theta(:,t)$ with the data points associated to that facet, that is, we take 
\[
q_t = \frac{\Theta(:,t)^\top X(:,J^{(t)}) e}{|J^{(t)}|} \; \text{ for } \; t=1,2,\dots,T. 
\] 
Finally, our estimation of the polytope $\conv\big( \tilde{W} \big)$ is given by 
$\mathcal{P} =  \{ x \ | \ \Theta^\top x \leq q \}$. 
Estimating $\tilde{W}$ from  $\mathcal{P}$ can be done using any off-the-shelf vertex enumeration algorithm. 
If $\rank(W) = r$, $\Theta \in \mathbb{R}^{(r-1) \times r}$, and each column of $\tilde{W}$ can be estimated by solving a linear system intersecting $r-1$ facets of $\mathcal{P}$. In the rank-deficient case, we have used the approach in \cite{bremner1998primal} whose 
implementation is provided in \cite{kleder2005con2vert}.

Finally, to estimate the matrix $W$, our estimated $\tilde{W}$ is projected back onto the original $m$-dimensional space, as in    Algorithm \ref{alg:bruteforce}.

%
%
%

\subsection{Identifiability, computational cost, and parameters} \label{discuss}

Algorithm~\ref{alg:gfpi} provides the pseudo-code for GFPI. The main difference with BFPI (Algorithm~\ref{alg:bruteforce}) is the way the facets of $\conv(\tilde{W})$ are extracted. Let us now discuss several important aspects of GFPI: 
its identifiability, 
the computational cost, and 
the choice of its parameters.   
\algsetup{indent=2em}
\begin{algorithm}[ht!]
	\caption{Greedy FPI (GFPI)} \label{alg:gfpi}
	\begin{algorithmic}[1]
		\REQUIRE Data matrix $X \approx WH \in \mathbb{R}^{m \times n}$ satisfying Assumption~\ref{ass1} approximately, 
		 number $T$ of facets to extract, 
		 dimension $d$, 
		 and the parameters $\gamma \geq 0$, $\eta > 0$, $\lambda > 0$, and $A > 0$.  
		
		\ENSURE Recover the basis matrix $W \in \mathbb{R}^{m \times r}$ approximately. 
		
		{\emph{\% Step 1. Preprocessing}} 
		
		\STATE Use the same preprocessing as in  Algorithm~\ref{alg:bruteforce}, to obtain $\tilde{X} = U^\top [X - \bar{X}] \in \mathbb{R}^{(d-1) \times n}$. 
		 
		 {\emph{\% Step 2. Extract the $T$ facets of $\conv\big( \tilde{W} \big)$}}  
		 
		\STATE Initialization: 
		Set $M^{(0)} = [\, ]$, 
		and $\Theta = [\, ]$.  
		
		\FOR{$t = 1, 2, \dots, T$}

		\STATE Compute $\theta^{(t)}$ as the optimal solution  of~\eqref{mixedinteger_fr_outlier}. 
		If $t = T = d$, use the additional constraint~\eqref{bounded} within~\eqref{mixedinteger_fr_outlier} to obtain a bounded polytope.  
		
		\STATE Identify the data points close to the facet corresponding to $\theta^{(t)}$, that is, 
		\[ 
		J^{(t)} = \left\{ j \ | \ \left| \tilde{X}(:,j)^\top \theta^{(t)} - 1 \right| \leq \gamma \right\}. 
		\]
		Note that $j \in J^{(t)}$ when $y_j = 0$ in~\eqref{mixedinteger_fr_outlier}.  
		
		\STATE Compute the average of these points as $m^{(t)} = \frac{\tilde{X}(:,J^{(t)})e}{|J^{(t)}|}$, 
		and let $M^{(t)} = [M^{(t-1)}, m^{(t)}]$. 
		
		\STATE Provide a more reliable estimate of $\theta^{(t)}$: 
		add as a column of $\Theta$ the left singular vector of \mbox{$\tilde{X}(:,J^{(t)}) - [m^{(t)} \dots m^{(t)} ]$} corresponding to its smallest singular value. 

       \STATE Compute the $t$th entry of the offset vector, $q_t = \frac{\Theta(:,t)^\top X(:,J^{(t)}) e}{|J^{(t)}|} = \Theta(:,t)^\top m^{(t)}$.

		\ENDFOR
		
		{\emph{\% Step 3. Recover $\tilde{W}$   }}

		\STATE Compute the columns $\tilde{W}$ as the $r$ vertices 
		of the polytope $\{ x \ | \ \Theta^\top x \leq q \}$. 
		
		If $T=d$, then $T=d=r$, and it is equivalent to solving the linear systems  $\Theta(:,\bar{k})^\top \tilde{W}(:,k) = q(\bar{k})$ for $k=1,2,\dots,r$ where $\bar{k} = \{1,2,\dots,r\} \backslash \{k\}$. 
		
		{\emph{\% Step 4. Postprocess $\tilde{W}$ to recover $W$}} 
		
		\STATE Project $\tilde{W} \in \mathbb{R}^{(d-1) \times r}$ back to the original $m$-dimensional space: $W =  U \tilde{W} + [\bar{x} \dots \bar{x}]$.  
		
	\end{algorithmic}
\end{algorithm}


\subsubsection{Identifiability in the noiseless case} 

For well-chosen parameters, GFPI recovers the unique SSMF under the FBC. 

\begin{theorem} \label{th:indentGFPI}
Let $X=WH$ satisfy the FBC (Assumption~\ref{ass1}). Let also the parameters of GFPI (Algorithm~\ref{alg:gfpi}) be as follows: 
$\gamma = 0$, 
$\eta$ is sufficiently small, 
$\lambda \rightarrow +\infty$, 
$A$ is sufficiently large, 
$T$ is the number of facets of $\conv(W)$, and $d = \rank(X)$. 
Then Algorithm~\ref{alg:gfpi} recovers the columns of $W$ (up to permutation).  
\end{theorem}
\begin{proof}
The preprocessing ensures that $0 \in \conv(\tilde{X})$ and $\tilde{X} \in \mathbb{R}^{(d-1) \times n}$ where $\rank(\tilde{X})=d-1$, while the geometry of the problem remains unchanged, as in Theorem~\ref{th:identFPI}. 

Let us discuss the parameters and their influence on~\eqref{mixedinteger_fr_outlier}: 
\begin{itemize}
    \item The variable $\delta$ was introduced to handle noise; see Section~\ref{sec:facetid}.  
    Taking $\lambda \rightarrow +\infty$ implies that the optimal solution for the variable $\delta$ in~\eqref{mixedinteger_fr_outlier} is 0, because $\delta = 0$ is part of many feasible solutions (take for example any $\theta$ such that $\tilde{X}^\top \theta \leq e$, such as $\theta = 0$ since $0 \in \conv(\tilde{X})$, and $y_j = 1$ for all $j$). In other words, in the noiseless case, $\delta$ can be set to zero and removed from the formulation~\eqref{mixedinteger_fr_outlier}. 
    Note that for $\delta = 0$, the first constraint of~\eqref{mixedinteger_fr_outlier} reduces to forming the dual space, that is, $\tilde{X}^\top \theta \leq e$, while the last constraints, dealing with outliers, can be removed since $A, y, \gamma \geq 0$. 
    
    \item For $A$ sufficiently large and $\gamma = 0$, the objective of~\eqref{mixedinteger_fr_outlier} is equivalent to the indicator function counting the points on the facet $\{ x \ | \ \theta^\top x = 1 \}$; see Section~\ref{sec:facetid}. 
\end{itemize}

This means that, for the chosen parameters, 
\eqref{mixedinteger_fr_outlier} is equivalent to  
\begin{equation} \label{identifnoiseless} 
\max_{\theta \in \mathbb{R}^{d-1}} \; \sum_{j=1}^{n}  \ 
\mathbf{I}\left(\tilde{X}(:,j)^\top\theta \geq 1 \right) \quad 
\text{ such that }  \quad 
\tilde{X}^\top \theta \leq e \; \text{ and } \;  {M^{(t-1)}}^\top \theta  \leq (1 - \eta) e. 
\end{equation} 

Now, let us prove the result by induction.  

\emph{First step.} Solving~\eqref{identifnoiseless} boils down to maximizing the number of data points in the set $\{ x \in \conv(\tilde{X}) \ | \ \theta^\top x = 1\}$. 
By Lemma~\ref{lem2}, this is a facet of $\conv(\tilde{W})$; in fact, it is a facet  containing the largest number of data points. 

\emph{Induction step.} Assume GFPI has extracted $k$ facets of $\conv(\tilde{W})$. 
The columns of $M^{(k)}$ are located in the relative interior of their corresponding facets. This follows from Assumption~\ref{ass1}.c because data points on that facet of $\conv(\tilde{W})$ generate that facet, and hence their average is in its relative interior. 
Because of the constraint ${M^{(k)}}^\top \theta  \leq (1 - \eta) e$, the previousy extracted $\theta^{(t)}$ ($1\leq t \leq k$) are eliminated from the solution space, that is, they are not feasible solutions of~\eqref{identifnoiseless}, because ${M^{(k)}}(:,t)^\top \theta^{(t)} = 1$ for $1\leq t \leq k$. 
Moreover, for $\eta$ sufficiently small, no other vertex of $\conv( \tilde{W} )^*$ is cut from the solution space (see Figure~\ref{margin} for an illustration). In fact, for $\eta \rightarrow 0$, only the vertices $\theta^{(t)}$ ($1\leq t \leq k$) are cut from $\conv(\tilde{X})^*$.   
Therefore the next step of GFPI  identiﬁes a facet of $\conv(\tilde{X})$ that is not extracted yet and that contains the largest possible number of data points. 
By Assumption~\ref{ass1}.c-d, this must correspond to a facet of $\conv( \tilde{W} )$.  
Note that, at the last step when $t=T$, if $T=d$, 
the constraint~\eqref{bounded} is added to~\eqref{identifnoiseless}. Since $\conv( \tilde{W} )$ is bounded, by definition, it does not prevent the model to extract the last facet of $\conv( \tilde{W} )$. It was used as a safety constraint in difficult scenarios; see Section~\ref{sec:bounded}. 
\end{proof}

In Section~\ref{sec:synth}, we will show that GFPI in fact performs perfectly in noiseless conditions under Assumption~\ref{ass1}.  
An important direction of research is to characterize the robustness to noise of GFPI. 
This is also an open problem for algorithms based on Min-Vol; see Section~\ref{relatedworks}.

\subsubsection{Computational cost} 

Identifying each facet requires to solve the MIP~\eqref{mixedinteger_fr_outlier}. 
Solving MIPs is in  general  NP-hard and can be time consuming.  In fact, the proposed model can be hard to solve up to global optimality when $n$ and/or $r$ become large.  Moreover, we have observed that, as the noise level increases, the problem gets more challenging which increases the computational time as well. 
We will use IBM-CPLEX (v12.10)~\cite{cplex2009v12} for solving the MIP~\eqref{mixedinteger_fr_outlier}. 
We noticed that CPLEX is able to find the optimal solution quite fast in many cases, even though it might require a lot of time to certify global optimality. 
Moreover, CPLEX is often able to find good feasible solutions quickly, and hence can be stopped early providing reasonable solutions for GFPI. In Section~\ref{sec:hsi}, we will use a time limit of 100 seconds for each facet identification on two large real data sets, and GFPI will provide solutions whose quality is similar to the state of the art.  
Interestingly, this observation holds even for problems with dimensions as large as 30. 
For example, in the noiseless case and for $d \leq 30$, CPLEX finds in most case the optimal solution for each facet in less than 100 seconds\footnote{For synthetic data sets, in the noiseless case, we know the optimal solution which allows us to check whether CPLEX found it. 
As shown in Appendix~\ref{appA:ccost}, for CPLEX to return the global optimal solution with a certificate takes more than one hour, even for small values of $r$ and $n$.}.  
In Appendix~\ref{appA:ccost}, we provide some additional numerical experiments on the computational cost of GFPI.  

A direction of further research would be to design dedicated algorithms (including heuristics) to tackle~\eqref{mixedinteger_fr_outlier}, taking advantage of its particular structure and geometry.

\subsubsection{Parameters} 

	 In GFPI, there are six parameters: $T$, $d$, $\gamma$, $\eta$, $\lambda$, and $A$. They were already discussed in the sections where they were introduced. 
	 Let us make two additional comments: 
	\begin{itemize}
	
	\item In most applications (for example in hyperspectral unmixing and topic modeling), we face full-dimensional problems, that is, $\rank(W) = r$, in which case $T=d=r$.  If it is not the case, a natural approach is to estimate 
 $d$ via the singular values of $X$, and stop the extraction of facets in $\Theta$ once the corresponding polyhedron $\{ x \ | \ \Theta^\top x \leq q \}$ is bounded and/or as long as enough data points are associated with the extracted facets.  
	
	\item Interestingly there is a way to select good values of the parameters $\gamma$, $\eta$ and $\lambda$ using trial and error. 
Even though GFPI looks for the facets sequentially, the final goal of GFPI is to \emph{identify enclosing facets containing the largest possible number of data points}. 
Hence, among different values of the parameters, 
one might select the ones that lead to the largest total number of data points on the $T$ extracted facets. 
		
		
		
	\end{itemize}
	
	In Appendix~\ref{appA:effectnumpts}, we provide additional numerical experiments analyzing the sensitivity of GFPI to its parameters. 
	




\section{Numerical Experiments} \label{numexp}

In this section, GFPI is evaluated on synthetic and real-world dat sets. 
All experiments are implemented in Matlab (R2019b), and run on a laptop with Intel Core i7-9750H, @2.60 GHz CPU and 16 GB RAM. 
We use IBM-CPLEX (v12.10)~\cite{cplex2009v12} for solving the MIP~\eqref{mixedinteger_fr_outlier}. The code is available from \url{https://sites.google.com/site/nicolasgillis/code}, and all experiments presented in this paper can be reproduced using this code. Note that the user can also use the Matlab MIP solver, \texttt{intlinprog},  which may be convenient. 


\paragraph{Compared Algorithms} 

GFPI is compared with the following state-of-the-art algorithms: 
\begin{itemize}

	\item Successive nonnegative projection algorithm (SNPA)~\cite{gillis2014successive}: This separable NMF algorithm is an extension of SPA. It is provably more robust to noise and can handle rank deficient matrices. 

	\item Simplex volume minimization: 
	we use the volume regularizer 
	$\logdet(W^\top W+\delta I_r)$ which has been shown to provide the best practical performances~\cite{fu2016robust, ang2019algorithms}. 
	We use the efficient algorithm based on block coordinate descent and the fast gradient method proposed in~\cite{leplat2019minimum}. It solves  
	\begin{equation} \label{eq:minvollogdet}
	\min_{W, H} \| X - WH \|_F^2 + \tilde{\lambda} \logdet(W^\top W+\delta I_r) \quad \text{ such that } H(:,j) \in \Delta^r \text{ for all } j. 
		\end{equation} 
 
	We will use different parameters for 
	$\tilde{\lambda} = \lambda 
	\frac{\| X - W^{(0)}H^{(0)} \|_F^2}{\logdet({W^{(0)}}^\top W^{(0)}+\delta I_r)}$
	where $(W^{(0)},H^{(0)})$ 
	is computed by SNPA; see~\cite{leplat2019minimum} for more details. 
	We refer to this algorithm as min vol.

	\item Maximum volume inscribed ellipsoid (MVIE)~\cite{lin2017maximum} is based on maximizing the volume of an ellipsoid within the simplex of data points; see Section~\ref{relatedworks} for more details.

	\item Hyperplane-based Craig-simplex-identiﬁcation (HyperCSI)~\cite{lin2015fast} estimates the $r$ purest samples using SPA, and calculates the hyperplanes of the enclosing simplex based on these samples; see Section~\ref{relatedworks} for more details.  
	
\end{itemize}

\paragraph{Quality measures} 

To quantify the performance of SSMF algorithms, the following metrics will be used. 
For the synthetic data experiments, we will use the relative distance between the ground-truth $W_t$ and the estimated $W$  
	\[
	\text{ERR} = \frac{||W_t \ - \ W||_F}{||W_t||_F}, 
	\]  
	where the columns of $W$ are permuted to minimize this quantity, using the Hungarian algorithm. 
	For real hyperspectral images, we will use the average mean removed spectral angle (MRSA) 
	between the columns of $W$ and $W_t$ (after a proper permutation of the columns of $W$). 
	This is the most common choice in this area of research. 
	The MRSA  between two vectors $x \in \mathbb{R}^n$ and $y \in \mathbb{R}^n$ is 
	\[
	\text{MRSA}(x,y) = \frac{100}{\pi} \cos^{-1} \bigg(\frac{(x-\bar{x} e)^\top (y-\bar{y} e)}{||x-\bar{x} e||_2||y-\bar{y} e||_2}\bigg)  , 
	\] 
	where $\bar{x} = \frac{1}{n} \sum_{i=1}^n x_i$. 
	We will also use the relative reconstruction error, RE $= \frac{||X-WH||_F}{||X||_F}$.

\subsection{Synthetic data sets} \label{sec:synth}

In this section, we compare GFPI with the state-of-the-art approaches on synthetic data sets. 

\paragraph{Data generation} 

To generate full-rank synthetic data sets $X = W_t H_t$, we follow a standard procedure; see for example~\cite{ang2019algorithms}. 
Each entry of $W_t$ is drawn uniformly at random from the interval $[0,1]$. We discard the matrices with condition number larger than $10r$ to avoid too ill-conditioned matrices. 

 We generate the columns of matrix $H_t$ by splitting them in two parts: $H_t = [H_1, H_2]$. 
 The matrix  $H_1$ corresponds to the points lying on facets, making sure there are enough points on each facet so that Assumption~\ref{ass1} holds. The matrix $H_2$ corresponds to data points randomly generated within $\conv(W)$. We generate $H_1$ and $H_2$ as follows. 

\begin{enumerate}
	\item Let $n_1$ be the number of data points on each facet. For each sample on a facet, the corresponding $r-1$ nonzero elements in the columns of $H_1$ are generated using the Dirichlet distribution with parameters 
	equal 
	to $\frac{1}{r-1}$.  
	For example, for $r=m=3$, $H_1$ has the following structure: 
	\[
	H_1=\left (
	\begin{array}{rrr|rrr|rrr}
	* & \dots & *   & *  & \dots & * &	0 &	\dots & 0\\
	* & \dots & *   & 0  & \dots & 0 & * & \dots & *\\
	\undermat{n_1}{0 & \dots & 0} & \undermat{n_1}{* & \dots & *} & \undermat{n_1}{* & \dots & *}\\\\ 
	\end{array}
	\right )
	\] 
	where `*' denotes nonzero elements generated using the Dirichlet distribution. 
	
	\item Let $n_2$ denotes the number of samples within the simplex, possibly lying on some facets but this is not strictly enforced. 
	The columns of $H_2$ are generated by the Dirichlet distribution with parameters set to $\frac{1}{r}$. 
	
	\end{enumerate} 
	
	Let us define the purity parameter $p \in (0,1]$ used to quantify how far the columns of $X$ are from the columns of $W_t$. 
	It is defined as $p(H_t) = \min_{1 \leq k \leq r} ||H_t(k,:)||_{\infty}$.  
 Recall that each row of $H_t$ corresponds to the activation of the corresponding column of $W$, while $H_t(:,j) \in \Delta^r$ for all $j$. Therefore, $p(H_t)$ indicates how much the separability assumption is violated. 
For $p(H_t)=1$, $X$ satisfies the separability assumption since each column of $W_t$ appears in the data set.  
For $p(H_t)=0$, at least one of the columns of $W$ is not used to generate $X$. 
	In order to control the purity of $H_t$, that is, $p(H_t)$, we use the parameter $p$, and 
	 resample the columns of $H_1$ and $H_2$ with entries larger than\footnote{
	To make the data generation possible, for $p \leq 0.3$, 
	we set the parameters of the Dirichlet distribution for the columns of $H_1$ to $\frac{1000}{r-1}$, otherwise most columns of $H_1$ are rejected.
	} $p$, that is, we define an upper bound on the entries of matrix $H_t$. 
	 Hence, using this resampling, $H_t(k,j) \leq p$ for all $k,j$ 
	 which implies $p(H_t) \leq p$. 
	 Note that $p$ has to be chosen larger than $\frac{1}{r-1}$ since 
	 $H(:,j) \in \Delta^r$ for all $j$, while the columns of $H_1$ have at least one zero entry.  

 Finally, the data matrix $X$ is generated by $X = W_t \ H_t$. In the presence of noise, we  use  additive Gaussian noise based on a given signal-to-noise ratio (SNR). 
 The variance of the i.i.d.\ random Gaussian noise given the SNR value is given by: 
\[ 
\text{ variance} \; = \; \frac{\sum_{i=1}^{m}\sum_{j=1}^{n}X_{i,j}^2}{10^{(SNR/10)} \times m \times n} . 
\]

\paragraph{Parameters for GFPI} 

The parameters of the proposed GFPI with respect to the noise level 
are selected according to Table~\ref{params}. 
As mentioned before, GFPI is not too sensitive to the parameter $\eta$ and we use 0.5 in all experiments. For the parameter $\lambda$, as it depends on the noise level, it should be decreased as the noise level increases; recall that $\lambda \rightarrow +\infty$ in the noiseless case (Theorem~\ref{th:indentGFPI}). 
The parameter $\gamma$ influences how the data points are associated to a facet: $X(:,j)$ is associated to the facet parametrized by $\theta$ when $|X(:,j)^\top \theta - 1| \leq \gamma$. 
Hence the larger the noise level, the larger $\gamma$ should be, since the data points will be moved further away from the facets.   
\begin{center}
	\begin{table*}[!htbp]
		\begin{center}
			\caption{Parameters of GFPI with respect to different values of SNR}
			\label{params}  
			\small\addtolength{\tabcolsep}{-1pt}
			\begin{tabular}{c||cccccc}
				\hline
				 & inf & 80 & 60 & 50 & 40 & 30 \\\hline
				$\lambda$ & 1000 & 100 & 100 & 10 & 10 & 10\\
				$\gamma$ & 0.001 & 0.01 & 0.01 & 0.05 & 0.1 &0.2 \\
				$\eta$ & 0.5 & 0.5 & 0.5 & 0.5 & 0.5 & 0.5\\
				\hline
			\end{tabular} 
		\end{center}
	\end{table*}
\end{center}

	For GFPI, 
	we have set the ``timelimit" property of CPLEX to 10 seconds. Whenever the upper bound on CPU time is activated, we specify it with ``**" after GFPI in the figures.

\subsubsection{Noiseless data sets} \label{noiseless_exp}

In this section, we investigate the effect of the purity on the performance of GFPI compared to the state-of-the-art approaches. 
To this end, we use the synthetic data with the following parameters: $n_1=30$ and $n_2=10$. 
Figure~\ref{noiseless} reports the average measure ERR over 10 randomly generated synthetic data sets obtained by the different algorithms for $r=m=\{3,4, 5,7\}$  as a function of the purity~$p$. 
  In this experiment, the value of the purity~$p$ varies between 
  $\frac{1}{r-1}+0.01$ (recall, $\frac{1}{r-1}$ is the smallest possible value) to 1 (separability). 
\begin{figure*}[htb]
	\begin{minipage}[b]{0.5\linewidth}
		\centering
		\centerline{\includegraphics[width=8cm]{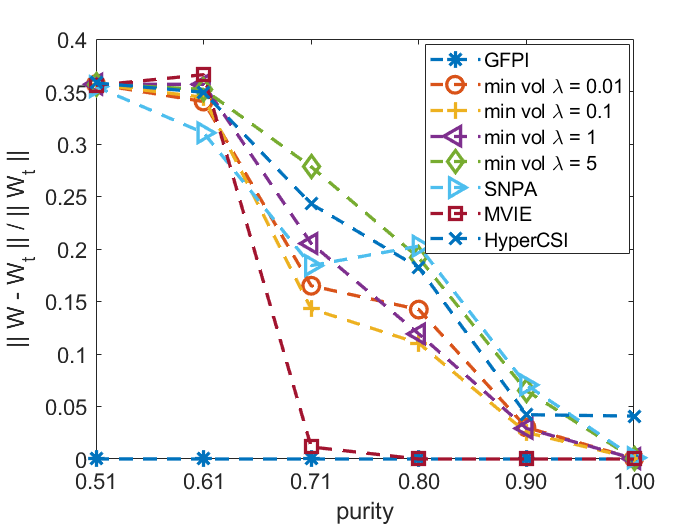}}
		\centerline{(a) $r=m=3$}\medskip
	\end{minipage}
	\hfill
	\begin{minipage}[b]{0.5\linewidth}
		\centering
		\centerline{\includegraphics[width=8cm]{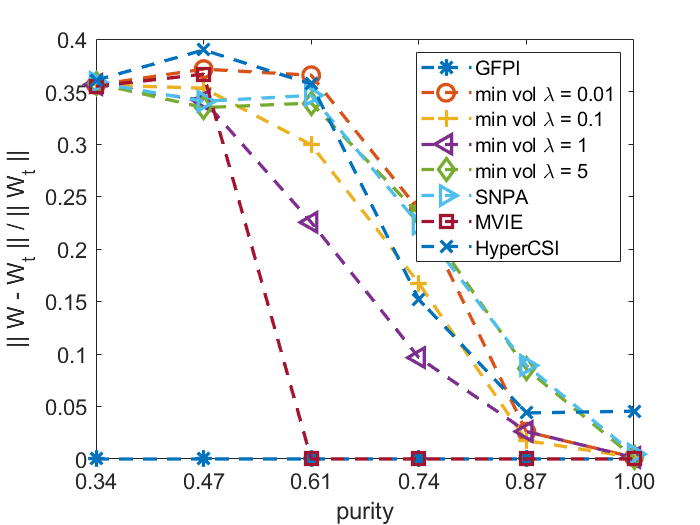}}
		\centerline{(b) $r=m=4$}\medskip
	\end{minipage}
	\hfill
	\begin{minipage}[b]{0.5\linewidth}
		\centering
		\centerline{\includegraphics[width=8cm]{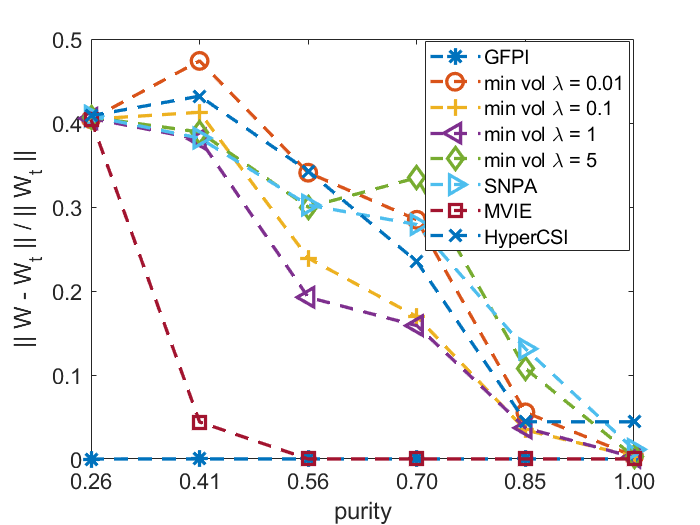}}
		\centerline{(c) $r=m=5$}\medskip
	\end{minipage}
	\hfill
	\begin{minipage}[b]{0.5\linewidth}
		\centering
		\centerline{\includegraphics[width=8cm]{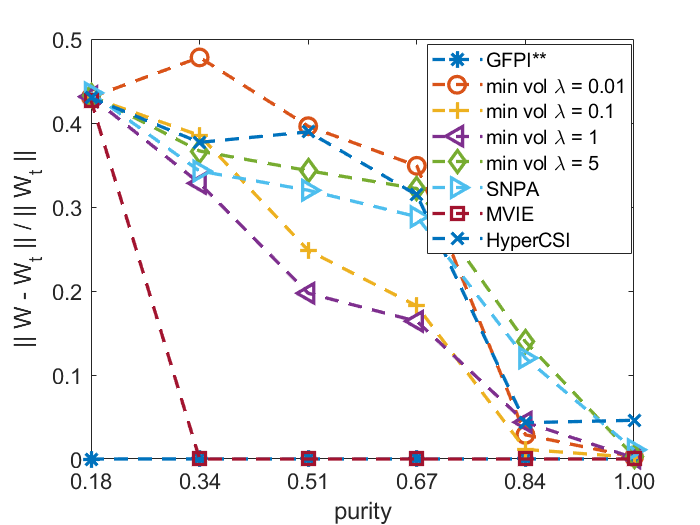}}
		\centerline{(d) $r=m=7$}\medskip
\end{minipage}
	\caption{Average ERR metric for 10 trials depending on the purity for SSMF algorithms in noiseless conditions for different values of $r$ and $m$. } 
	\label{noiseless}
\end{figure*}

GFPI recovers $W_t$ perfectly for all cases, and the performance is not dependent on the purity, as expected  since Assumption~\ref{ass1} is satisfied, regardless of the purity (Theorem~\ref{th:indentGFPI}).  
On the other hand, the performance of all other approaches gradually decreases as the purity decreases. For SNPA (which is based on the separability assumption), the performance worsens as soon as $p < 1$. 
For low levels of purity, the SSC is not satisfied, and hence the performances of min vol and MVIE degrade as $p$ decreases. 
In fact, it is interesting to observe that MVIE performs perfectly for $p$ sufficiently large, when the SSC is satisfies (as guaranteed by the theory), while min vol degrades its performances faster as it relies on heuristics and is sensitive to initialization.   
A similar behavior was already observed in~\cite{lin2017maximum}.

\subsubsection{Noisy data sets} \label{noisy_exp}

In this section, we compare the behavior of the different algorithms in the presence of noise. 
We use three levels of noise (SNR = 60, 50 and 40) and investigate the effect of the purity for $r=m=\{3,4\}$. 
Figure~\ref{noisy1} reports the ERR metric, similarly as for Figure~\ref{noiseless} (average of 10 randomly generated synthetic data sets). 
\begin{figure*}[htp]
	\begin{minipage}[b]{0.5\linewidth}
		\centering
		\centerline{\includegraphics[width=8cm]{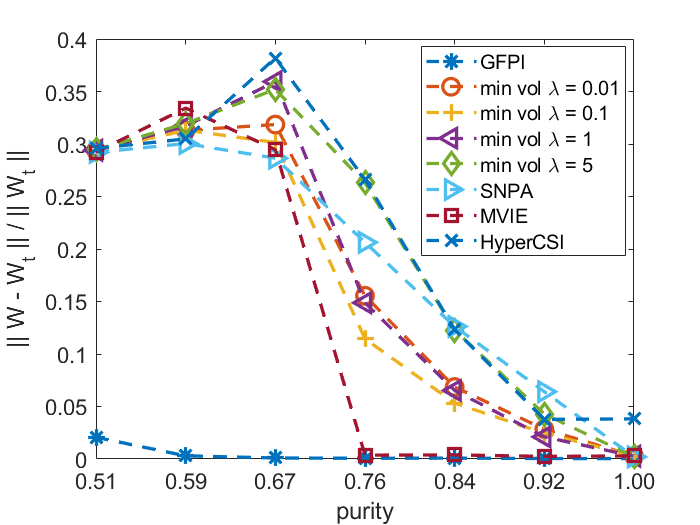}}
		\centerline{(a) $r=m=3$, SNR = 60}\medskip
	\end{minipage}
	\hfill
	\begin{minipage}[b]{0.5\linewidth}
		\centering
		\centerline{\includegraphics[width=8cm]{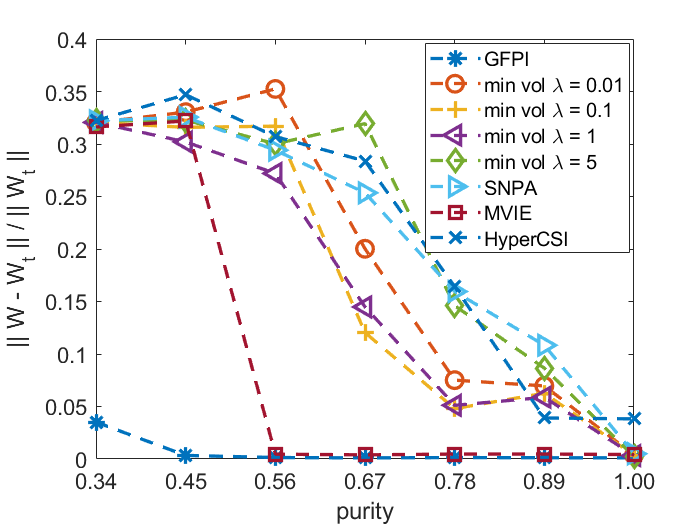}}
		\centerline{(b) $r=m=4$, SNR = 60}\medskip
	\end{minipage}
	\hfill
	\begin{minipage}[b]{0.5\linewidth}
		\centering
		\centerline{\includegraphics[width=8cm]{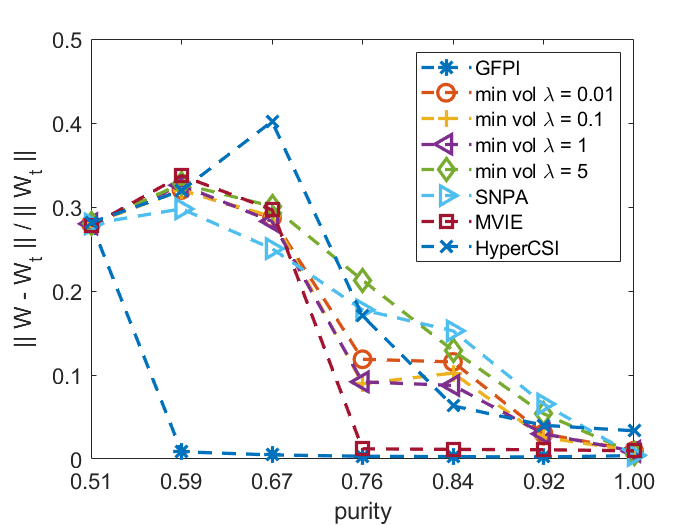}}
		\centerline{(c) $r=m=3$, SNR = 50}\medskip
	\end{minipage}
	\hfill
	\begin{minipage}[b]{0.5\linewidth}
		\centering
		\centerline{\includegraphics[width=8cm]{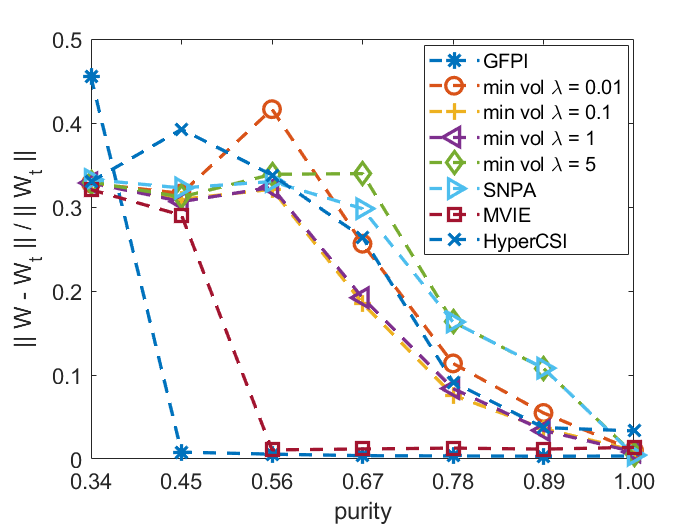}}
		\centerline{(d) $r=m=4$, SNR = 50}\medskip
	\end{minipage}
	\hfill
	\begin{minipage}[b]{0.5\linewidth}
		\centering
		\centerline{\includegraphics[width=8cm]{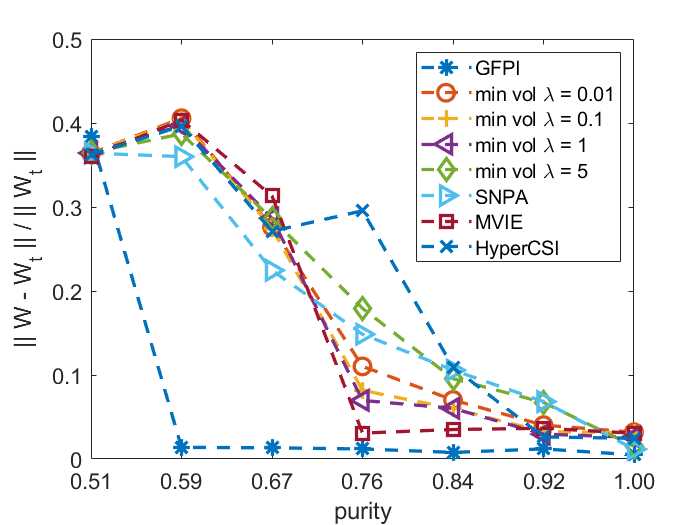}}
		\centerline{(e) $r=m=3$, SNR = 40}\medskip
	\end{minipage}
	\hfill
	\begin{minipage}[b]{0.5\linewidth}
		\centering
		\centerline{\includegraphics[width=8cm]{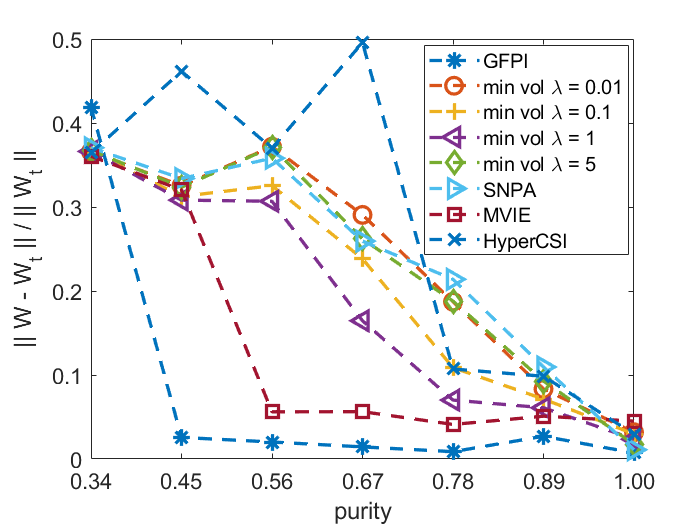}}
		\centerline{(f) $r=m=4$, SNR = 40}\medskip
	\end{minipage}
	\caption{Average ERR metric for 10 randomly generated data sets depending on purity for the different SSMF algorithms, for different noise levels: SNR of 60 (top), 50 (middle) and 40 (bottom), 	and for  $m=r=3$ (left) and $m=r=4$ (right).}
	\label{noisy1}
\end{figure*} 
As the noise level increases (SNR decreases), the performance of all algorithms decreases steadily. However, in almost all cases, GFPI outperforms all other approaches, especially when the  the purity $p$ is low. 
As for the noiseless case, MVIE performs the second best.

To further understand the performance of GFPI in presence of noise and under different purity levels, Figure~\ref{noisy2} reports the average ERR metric over 10 trials for very low purity values (namely $p \in [1/(r-1)+0.01,1/(r-1)+0.1]$) depending on the noise level for $m=r=\{3,4\}$. 
 Note that all other algorithms fail in this range of purity; see Figure \ref{noisy1}. 
 \begin{figure*}[htp]
	\begin{minipage}[b]{0.5\linewidth}
		\centering
		\centerline{\includegraphics[width=7cm]{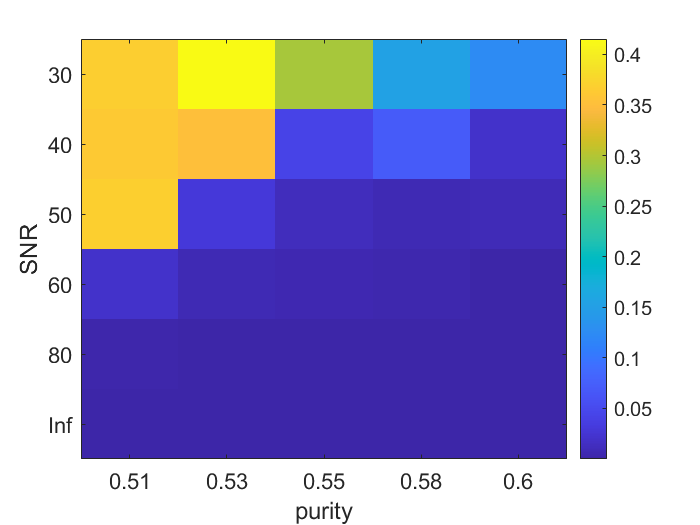}}
		\centerline{(a) $r=m=3$}\medskip
	\end{minipage}
	\hfill
	\begin{minipage}[b]{0.5\linewidth}
		\centering
		\centerline{\includegraphics[width=7cm]{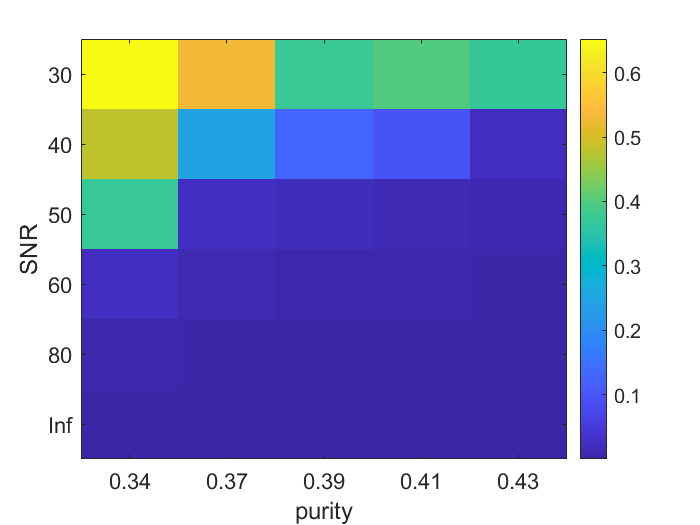}}
		\centerline{(b) $r=m=4$}\medskip
	\end{minipage}
	\caption{ERR values depending on the purity $p$ and the SNR  
	for GFPI.}
	\label{noisy2}
\end{figure*} 
We observe that for SNR $=\infty$ (noiseless case), the performance is independent of the value of purity. 
For $\text{SNR} \leq 50$, ERR gradually increases as the SNR decreases. 
As $p$ decreases, the data points are located closer to the center of the facets, and hence it is more challenging to recover the  facets in the presence of noise. Note also that, quite naturally, the robustness of GFPI depends on the number of points per facet; see Appendix~\ref{appA:effectnumpts}.

\subsubsection{Rank-deficient SSMF} \label{sec:rankdef}

An advantage of GFPI is that it provably works when $W$ does not have full column rank, 
and without the separability assumption. 
Note that 
\begin{itemize}

    \item SNPA works in the rank-deficient case, but requires the separability assumption. 
    Other separable NMF algorithms also work in the rank-deficient case; for example~\cite{arora2012computing, recht2012factoring, gillis2014robust} but are computationally much more demanding than SNPA as they rely on solving $n$ linear programs in $n$ variables.

    \item The min-vol model ~\eqref{eq:minvollogdet} can be used in the rank-deficient case~\cite{leplat2019minimum}. However, it does not come with identifiability guarantees (this is actually an open problem). 
    
\end{itemize}
MVIE and HyperCSI are not applicable when $\rank(W) < r$. 

In this section, we confirm the ability of GFPI to recover $W$ when it does not have full column rank. 
To do so, we use the rank-deficient synthetic data from~\cite{leplat2019minimum}. 
It generates the matrix $X \in \mathbb{R}^{4 \times 200}$ using the rank-deficient matrix 
\[
W_t = \left (
\begin{array}{rrrr}
1 & 1 & 0 & 0\\
0 & 0 & 1 & 1\\
0 & 1 & 1 & 0\\
1 & 0 & 0 & 1
\end{array}
\right ), 
\]
for which $\rank(W_t)=3 < r=4$. 
Each column of $H_t \in \mathbb{R}^{4 \times 200}$ is generated using the Dirichlet distribution with parameters equal to 0.1. The columns of $H$ with elements larger than a predefined purity value $p$ are resampled, as before. 
In this experiment, we consider three values for the purity, namely 0.8, 0.7 and 0.6. 
We take $X = W_t H_t$ and then corrupt it with i.i.d.\ 
Gaussian distribution with zero mean and standard deviation set to 0.01. 
GFPI parameters are $\lambda=10$, $\eta = 0.5$, $\gamma=0.05$, and $A=10$. 
Note that $\lambda$ is relatively large since there are not outliers and the noise level is low. 

Figure~\ref{rankdeficient} shows the result, after projection of the data points in two dimensions. 
Table~\ref{rankdef_tab} reports the ERR metric for the different algorithms.  
\begin{figure*}[htb]
	\begin{minipage}[b]{0.3\linewidth}
		\centering
		\centerline{\includegraphics[width=6.5cm]{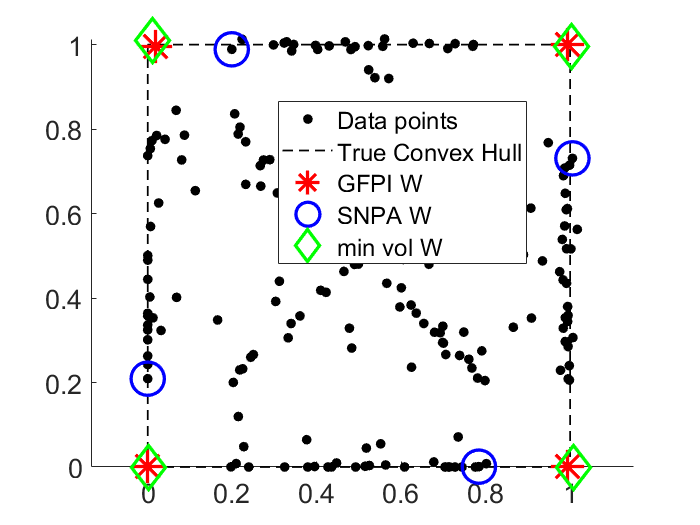}}
		\centerline{(a) purity = 0.8}\medskip
	\end{minipage}
	\hfill
	\begin{minipage}[b]{0.3\linewidth}
		\centering
		\centerline{\includegraphics[width=6.5cm]{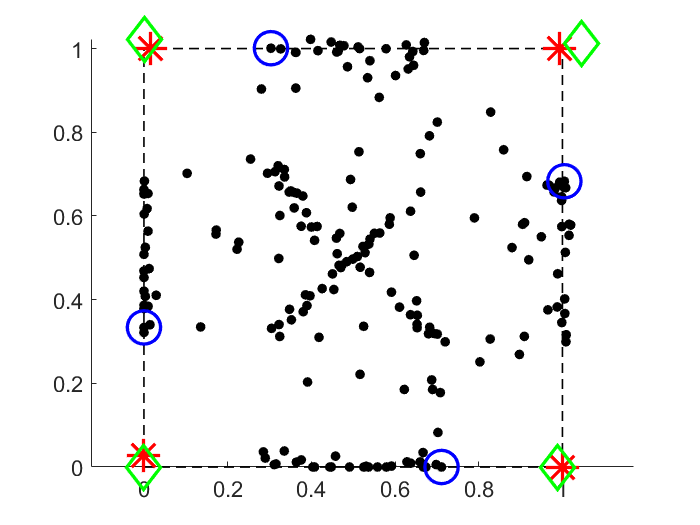}}
		\centerline{(b) purity = 0.7}\medskip
	\end{minipage}
	\hfill
	\begin{minipage}[b]{0.3\linewidth}
		\centering
		\centerline{\includegraphics[width=6.5cm]{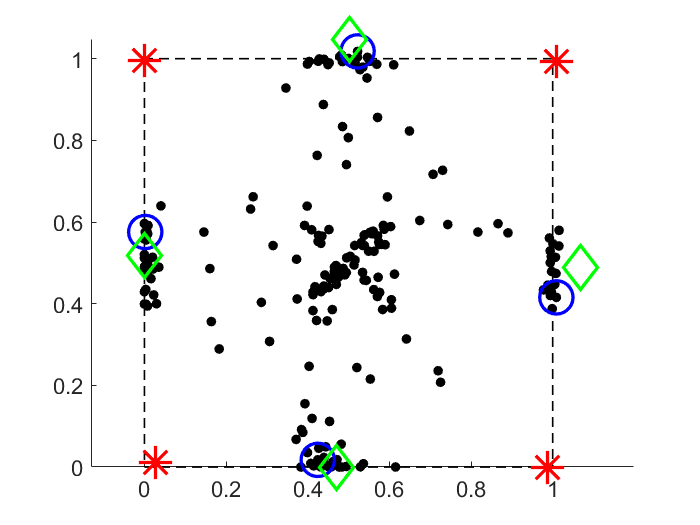}}
		\centerline{(c) purity = 0.6}\medskip
	\end{minipage}
	\caption{Two dimensional representation of the estimated vertices in rank-deficient cases with different values of purity. }
	\label{rankdeficient}
\end{figure*}
\begin{center}
	\begin{table*}[!ht]
		\begin{center}
			\caption{
			Comparing ERR of SNPA, min vol, and GFPI in dealing with synthetic rank-deficient data. 
			\label{rankdef_tab}} 
			 
			\small\addtolength{\tabcolsep}{-1pt}
			\begin{tabular}{c|ccc}
				\hline
				purity & SNPA & min vol & GFPI \\ \hline
				0.8 & 0.219 & 0.014 & 0.010 \\
				0.7 & 0.315 & 0.029 & 0.018 \\
				0.6 & 0.429 & 0.485 & 0.017 \\
				\hline				
			\end{tabular} 
		\end{center}
	\end{table*}
\end{center}

Since the data is not separable, SNPA provides the worst solutions. 
For $p \in \{0.7, 0.8\}$, min vol performs well, although slightly worse than GFPI. 
For $p = 0.6$, min vol fails to extract columns of $W_t$, as the purity is not large enough. However, it recovers a reasonable solution with smaller volume; this is a similar behavior as in Figure~\ref{unique}.

\subsubsection{Performance in the presence of outliers} \label{sec:outliersexp}

In this section we investigate the ability of GFPI to deal with outliers. As mentioned earlier, as far as we know, most SSMF algorithms 
are very sensitive to outliers (in particular, most separable NMF algorithms, min vol, MVIE and HyperCSI).  
To do so, we generate the clean data by considering $m=r=3$, $p=1$ 
(no resample of the columns of $H_t$ so $p(H_t)$ is close to 1), 
$n_1 = 30$, $n_2 = 10$ data points (for a total of 100 clean samples), and SNR = $\infty$. 
We then add outliers whose entries are drawn from the uniform distribution in $[0,1]$.
GFPI parameters are $\lambda=0.01$, $\eta = 0.5$, $\gamma=0.01$, and $A=100$.  
The parameter $\lambda$ is chosen relatively small allowing $\delta$ to take larger values, which is necessary in the presence of outliers. 

Figure~\ref{outlierex} reports the results on four different examples, with 3, 10, 50 and 100 outliers (red crosses). 
It shows the columns of $W$ and their corresponding convex hulls estimated by the different algorithms. 
In all cases, GFPI perfectly recovers the true endmembers, while the other algorithm fail. 
In fact, even few outliers affects their performance whereas  GFPI tolerates as many outliers as the number of clean samples.  The reason for this robustness to outliers is that outliers are generated randomly, and hence no more than $d-1$ outliers belong to the same hyperplane (with probability one);  
in this example, no combination of three outliers belong to the same segment.  
Of course, adding adversarial outliers on the same hyperplane would lead to different results. However, as long as the number of outliers on the same hyperplane is smaller than the number of points on the facets of $\conv(W)$, GFPI will perform well. 

\begin{figure*}[ht!]
	\begin{minipage}[b]{0.5\linewidth}
		\centering
		\centerline{\includegraphics[width=\textwidth]{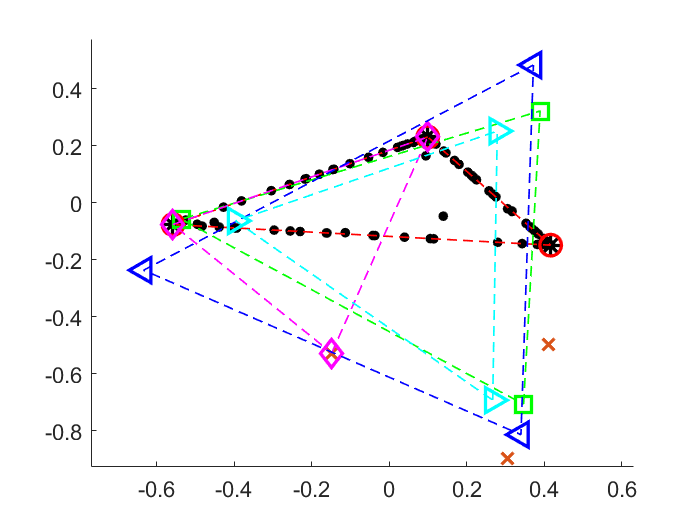}}
		\centerline{(a) 3 outliers}\medskip
	\end{minipage}
	\hfill
	\begin{minipage}[b]{0.5\linewidth}
		\centering
		\centerline{\includegraphics[width=\textwidth]{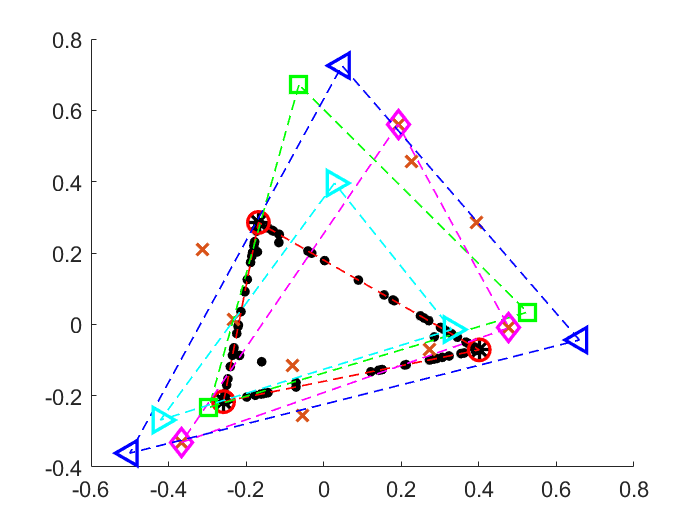}}
		\centerline{(b) 10 outliers}\medskip
	\end{minipage}
	\hfill
	\begin{minipage}[b]{0.5\linewidth}
		\centering
		\centerline{\includegraphics[width=\textwidth]{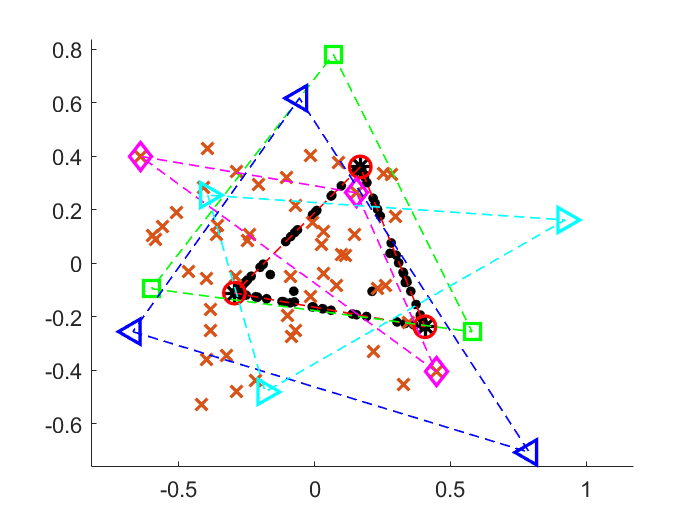}}
		\centerline{(c) 50 outliers}\medskip
	\end{minipage}
	\hfill
	\begin{minipage}[b]{0.5\linewidth}
		\centering
		\centerline{\includegraphics[width=\textwidth]{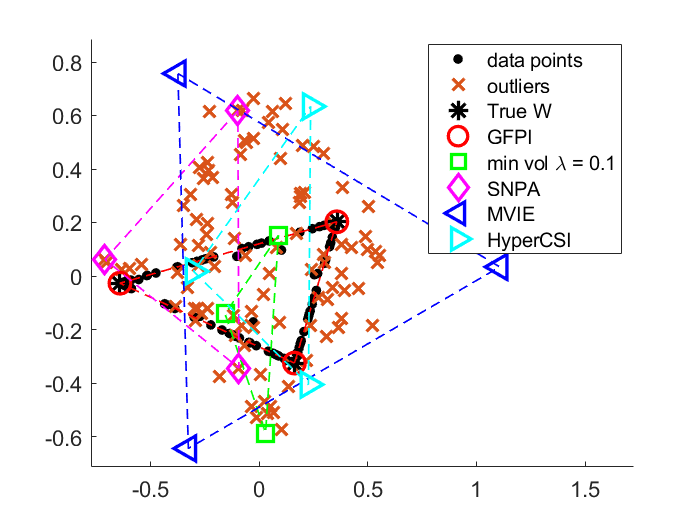}}
		\centerline{(d) 100 outliers}\medskip
	\end{minipage}
	\caption{Comparison of SSMF algorithms in the presence of outliers.}
	\label{outlierex}
\end{figure*}

\newpage 

\subsection{Hyperspectral images} \label{sec:hsi}

In this section, we evaluate the performance of GFPI on two widely used hyperspectral images, namely Samson and Jasper Ridge; see~\cite{zhu2017hyperspectral} and the references therein. 
These hyperspectral images are relatively large, containing thousands of pixels. 
Hence we set the \emph{timelimit} of CPLEX for optimizing each facet  to 100 seconds. 
We will provide the MRSA for the extracted factors by the different SSMF algorithms. It is important to note that the ground truth factor $W_t$ are actually unknown, and these estimates come from~\cite{zhu2017hyperspectral}. 
Moreover, the reported result for min vol are the best possible performance with highly tuned parameters from~\cite{ang2019algorithms}. 
Once the matrix $W$ is estimated, we estimate the matrix $H$ by solving  
\begin{align} \label{obtainH}
\min_{H \in \mathbb{R}^{r \times n}} ||X-WH||_F^2 \quad \text{ such that } \quad H(:,j) \in \Delta^r \text{ for all } j, 
\end{align}
which is a convex linearly constrained least squares problem.  We use the code from~\cite{gillis2014successive}.

\subsubsection{Samson}

The Samson data set consists of $95 \times 95$  images for 156 spectral bands~\cite{zhu2017hyperspectral}. 
Mostly three materials are present in this image: ``soil", ``water" and ``tree", and hence $r=3$. 
We run GFPI to extract three endmembers with parameters: 
$T=d=3$, 
$\lambda=0.1$, $\gamma=0.3$, $\eta=0.7$ and $A = 10$.  

The extracted spectral signatures and the corresponding abundance maps are shown in Figure~\ref{samson_W}~(a) and~(c), respectively. 
To interpret GFPI geomerically, 
Figure~\ref{samson_W}~(b) shows the data points and the polytope computed by GFPI, projected onto a two-dimensional subspace spanned by the first two components of the PCA of the input matrix.  

\begin{figure*}[ht!]
	\begin{minipage}[b]{0.5\linewidth}
		\centering
		\centerline{\includegraphics[width=8cm]{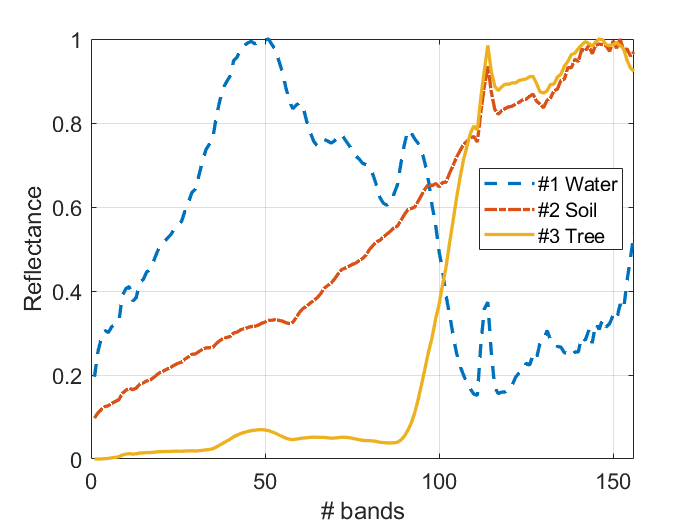}}
		\center{(a) Estimated spectral signatures of the three endmembers, that is, columns of $W$.}\medskip
	\end{minipage}
	\hfill
	\begin{minipage}[b]{0.5\linewidth}
		\centering
		\centerline{\includegraphics[width=8cm]{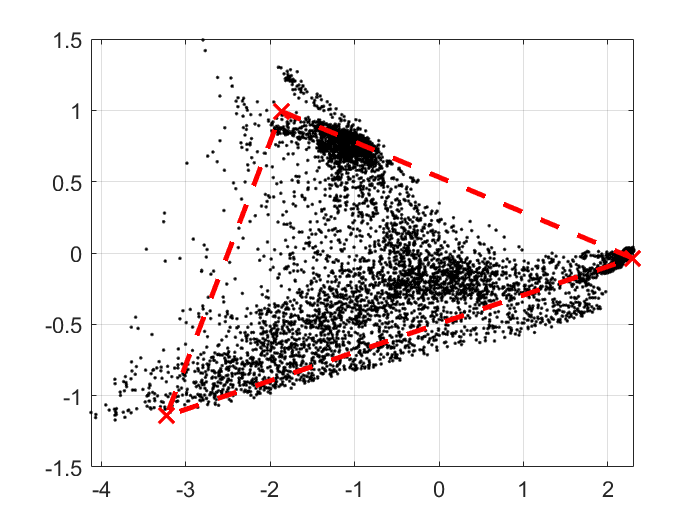}}
		\center{(b) Projection onto a two-dimensional subspace of the data points (dots), and the polytope computed by GFPI (crosses connected by dashed lines). }\medskip
	\end{minipage}
	\hfill
	\begin{minipage}[b]{1\linewidth}
		\centering
		\centerline{\includegraphics[width=15cm]{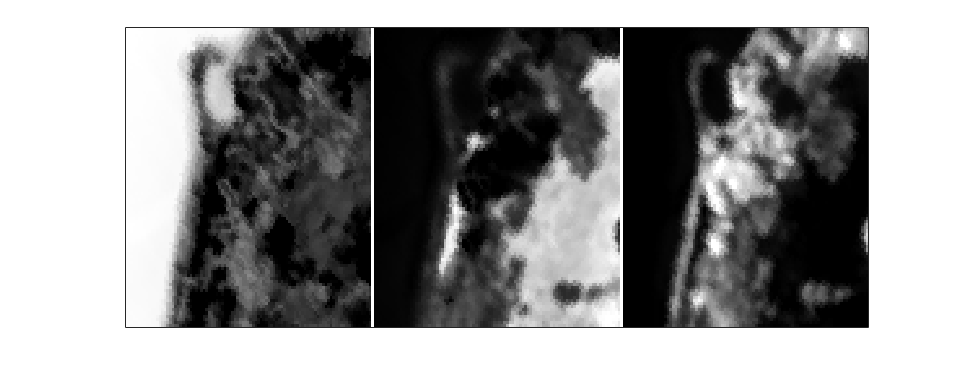}}
		\center{(c) Abundance maps corresponding to the estimated $W$ by GFPI. 
		Each abundance map corresponds to a reshaped row of $H$. From left to right: water, soil and tree. 
		}\medskip
	\end{minipage}
	\caption{GFPI applied on the Samson hyperspectral image. 
	\label{samson_W}
	} 
\end{figure*}

Table~\ref{samsont} reports the MRSA and RE for GFPI, SNPA, min vol, and HyperCSI. 
Note that MVIE is computationally too expensive and we excluded it from the comparison. 
GFPI performs similarly to SNPA and slightly worse than min vol. 
HyperCSI has the worst performance among the four.  
However, this illustrates that CPLEX finds {good} feasible solutions for the proposed MIP relatively fast. 
\begin{center}
	\begin{table*}[!htbp]
		\begin{center}
			\caption{Comparing the performances of GFPI with HyperCSI, SNPA and min vol on Samson data set}
			\label{samsont}  
			\small\addtolength{\tabcolsep}{-1pt}
			\begin{tabular}{c||ccccc}
				\hline
				& SNPA  & min vol & HyperCSI & GFPI \\\hline
				MRSA & 2.78 & 2.58 & 12.91 & 2.97 \\
				$\frac{||X-WH||_F}{||X||_F}$ & 4.00\% & 2.69\% & 5.35\% & 4.02\% \\
				\hline
			\end{tabular} 
		\end{center}
	\end{table*}
\end{center}

\subsubsection{Jasper Ridge} 

The Jasper Ridge data set consists of $100 \times 100$  images for 224 spectral bands~\cite{zhu2017hyperspectral}. 
Mostly four materials are present in this image: ``road", ``soil", ``water" and ``tree". 
We run GFPI to extract four endmembers with parameters: 
$T=d=4$, 
$\lambda=0.0001$, $\gamma=0.2$, $\eta=0.5$ and $A = 10$.  
Note that $\lambda$ is rather small, much smaller than for Samson ($\lambda=0.1$). 
Because such data sets are very noisy and violate the model assumptions, 
GFPI is more sensitive to its parameters which should be carefully tuned (note that it is also sensitive to the time limit used in CPLEX, and hence to the power of the computer it is run on).  
However, although we have fine-tuned GFPI parameters for these real-world experiments, it provides good solutions for a different values of the parameters. For example, we also obtain reasonable solutions for $\lambda \in [0.01, 0.0001]$.  

 The extracted spectral signatures and the corresponding abundance maps are shown in Figure~\ref{jasper_W}~(a) and~(c), respectively. 
 Similar to the Samson data set, 
 the two dimensional representation of the data points and the estimated polytope are shown in Figure~\ref{jasper_W}~(b). 
 
 \begin{figure*}[ht!]
	\begin{minipage}[b]{0.5\linewidth}
		\centering
		\centerline{\includegraphics[width=8cm]{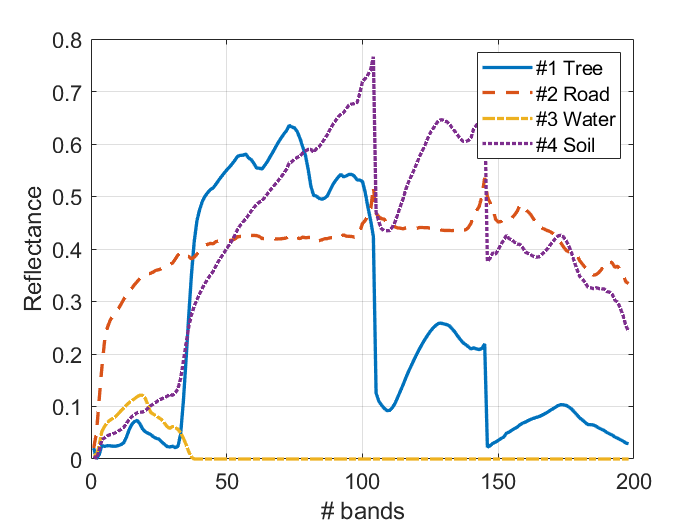}}
		\center{(a) Estimated spectral signatures of the four endmembers, that is, columns of $W$.}\medskip
	\end{minipage}
	\hfill
	\begin{minipage}[b]{0.5\linewidth}
		\centering
		\centerline{\includegraphics[width=8cm]{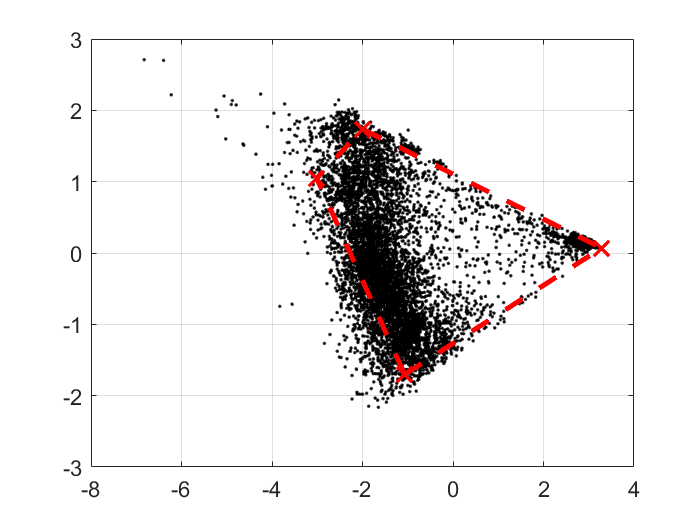}}
		\center{(b) Projection onto a two-dimensional subspace of the data points (dots), and the polytope computed by GFPI (crosses connected by dashed lines). }\medskip
	\end{minipage}
	\hfill
	\begin{minipage}[b]{1\linewidth}
		\centering
		\centerline{\includegraphics[width=17cm]{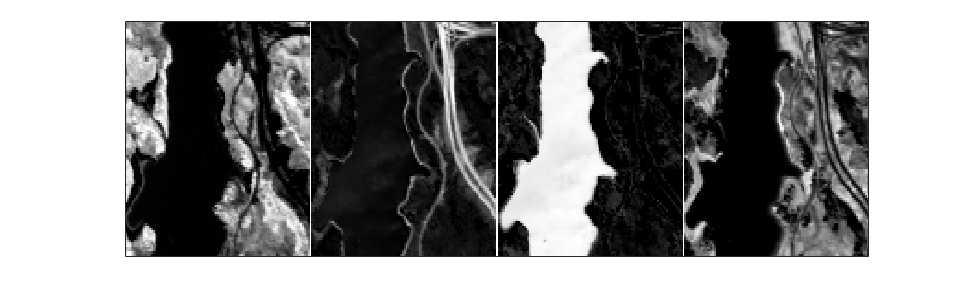}}
		\center{(c) Abundance maps corresponding to the estimated $W$ by GFPI. 
		Each abundance map corresponds to a reshaped row of $H$. From left to right: 
		tree, road, water, soil. 
		}\medskip
	\end{minipage}
	\caption{GFPI applied on the Jasper ridge hyperspectral image. 
	\label{jasper_W}
	} 
\end{figure*}

 Table~\ref{jaspert} reports the MRSA and RE. We observe that GFPI has the lowest (best) MRSA value and second best RE among the four algorithms. 
\begin{center}
	\begin{table*}[!htbp]
		\begin{center}
			\caption{Comparing the performances of GFPI with HyperCSI, SNPA and min vol on Jasper database}
			\label{jaspert}  
			\addtolength{\tabcolsep}{-1pt}
			\begin{tabular}{c||ccccc}
				\hline
				& SNPA  &  min vol & HyperCSI & GFPI \\\hline
				MRSA & 22.27 & 6.03 & 17.04 & 4.82 \\
				$\frac{||X-WH||_F}{||X||_F}$ & 8.42\% & 6.09\% & 11.43\% & 6.47\% \\
				\hline
			\end{tabular} 
		\end{center}
	\end{table*}
\end{center}

Note that it is natural for min vol to have the lowest RE as it is part of its objective function. Having a low RE for GFPI is a side result of $W$ being well estimated. In particular, GFPI is able to discard outliers (see Section~\ref{sec:outliersexp}) which may increase the RE significantly because this measure is very sensitive to outliers (least squares).  
Once $W$ is estimated by GFPI, the RE, or other quality measures, could be used to assess whether GFPI provided a reasonable solution (in fact, GFPI never uses this quantity as a criterion for estimating $W$). This would be another way to fine tune the parameters of GFPI.

\section{Conclusion} \label{conc}

In this paper, we have presented a new framework for simplex-structured matrix factorization (SSMF). The high level idea is to identify the facets of the convex hull of the basis matrix $W$ by looking for facets of the convex hull of the data matrix $X = WH$ containing the largest number of points. 
We first proved that under our facet-based conditions (FBC, see Assumption~\ref{ass1}), SSMF is identifiable, that is, it has a unique solution $W$, up to permutation of the columns (Theorem~\ref{th:identifSSMFFBC}). 
Then, we proposed and analyzed brute-force facet-based polytope identification (BFPI) 
which converts the problem of searching for the facets to the problem of identifying the vertices in the dual space.  BFPI recovers the ground truth $W$ under the FBC  (Theorem~\ref{th:identFPI}). 
We also proposed GFPI (greedy FPI) which sequentially identifies the facets (instead of identifying them all) using MIPs, and comes with identifiabiliy guarantees (Theorem~\ref{th:indentGFPI}). In order to handle noise and outliers, we have proposed a very effective MIP to tackle the subproblem for identifying a facet. We have also proposed an effective postprocessing step to improve the recovery of $W$ by reestimating the facets using the data points associated to them.  
We illustrated the effectiveness of GFPI compared to state-of-the-art SSMF algorithms. 
 GFPI is able to handle highly mixed data points for which the conditions under which the other algorithm work are highly violated (namely, separability and the SSC). It is also able to handle many outliers, and rank-deficient matrices $W$. We also provided encouraging numerical experiments on real-world hyperspectral images. GFPI is applicable to large data sets because the MIPs do not need to be solved up to global optimality: any solution returned by the solver can be used by GFPI to construct a facet.

Directions of further research include the 
identifiability of GFPI in presence of noise and outliers, 
the design of more effective MIP formulations to identify the facets, 
and the improvement of the scalability of GFPI for large-scale data sets (for example by designing dedicated algorithms to solve the MIPs).

\small

\bibliographystyle{spmpsci} 
\bibliography{sparseNMF} 

\normalsize 

\newpage 

\appendix

\section{Appendix} \label{appA}

In this section, we provide additional experiments to further investigate properties of GFPI.

\subsection{Effect of the number of data points on the facets} \label{appA:effectnumpts} 

In this section, we study the effect of number of points per facet of $\conv(W)$ on GPFI for different noise levels for two different purity values, namely $0.6$ and $1$. The number of points on each facet is $\{r, 2\times r, 4\times r, 10 \times r, 20 \times r\}$. The experiments are carried out for $m=r=\{3,4\}$. The average ERR for 10 trials are reported in 
Figure~\ref{numberpoints}. 

We observe that, 
as the SNR decreases, more data points on each facet improves the solution quality. This makes sense: as the number of points per facet increases, the robustness to noise of GFPI is improved. 


\begin{figure*}[htb]
	\begin{minipage}[b]{0.5\linewidth}
		\centering
		\centerline{\includegraphics[width=7cm]{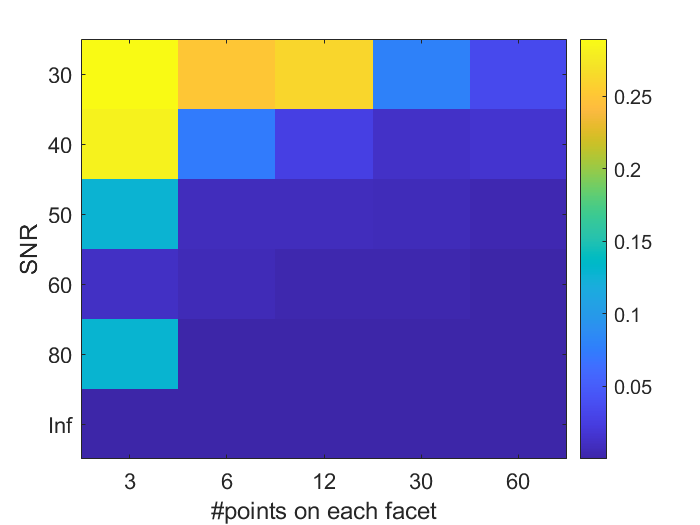}}
		\centerline{(a) $r=m=3$, purity = 0.6}\medskip
	\end{minipage}
	\hfill
	\begin{minipage}[b]{0.5\linewidth}
		\centering
		\centerline{\includegraphics[width=7cm]{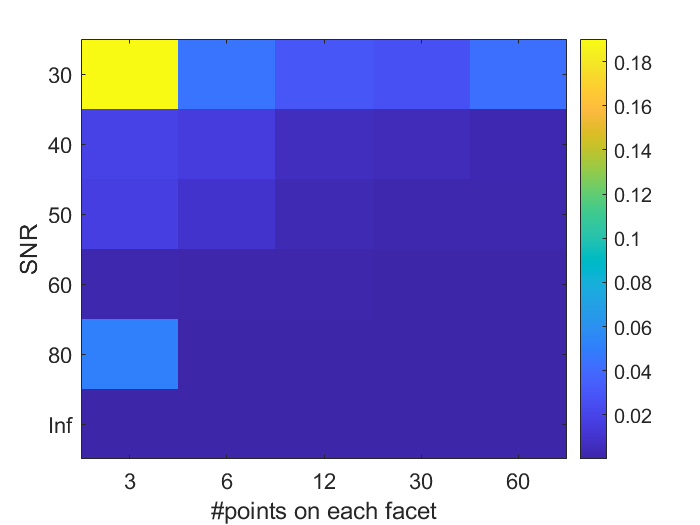}}
		\centerline{(b) $r=m=3$, purity = 1}\medskip
	\end{minipage}
	\hfill
	\begin{minipage}[b]{0.5\linewidth}
		\centering
		\centerline{\includegraphics[width=7cm]{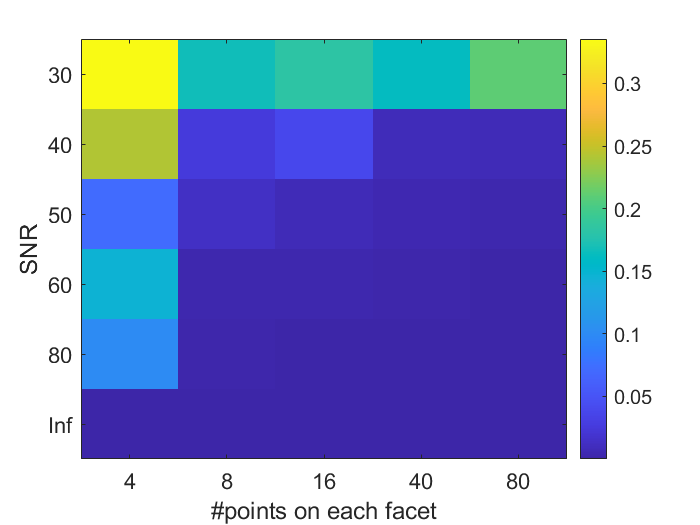}}
		\centerline{(c) $r=m=4$, purity = 0.6}\medskip
	\end{minipage}
	\hfill
	\begin{minipage}[b]{0.5\linewidth}
		\centering
		\centerline{\includegraphics[width=7cm]{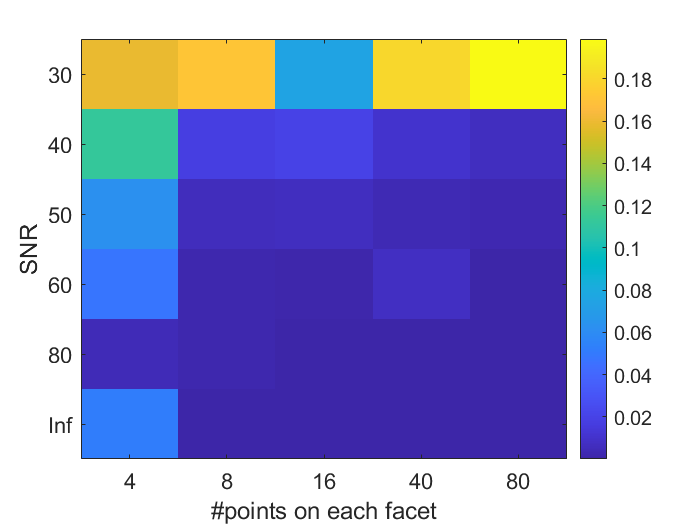}}
		\centerline{(d) $r=m=4$, purity = 1}\medskip
	\end{minipage}
	\caption{The Effect of number of points per facets on the ERR metric.}
	\label{numberpoints}
\end{figure*}

\subsection{Computational cost of GFPI} \label{appA:ccost}

The proposed GFPI is based on optimizing a MIP 
 for identifying each facet which can be time consuming. 
 In this section, we report the computational time for the proposed approach with respect to different parameters: 
 the number of points on each facet ($n_1$), $r$ ($m$), 
 the noise level, and the purity. 
 We consider 5 different values for $n_1$ and report the run time of our approach under different settings for other parameters. Table~\ref{tab:runningtime} reports the average running time (CPU time in seconds) for three trials for  $\{p = 1, \text{SNR} = \infty\}$, 
 $\{p = 0.6, \text{SNR} = \infty\}$ , 
 $\{p = 0.6, \text{SNR} = 40\}$, and 
 $\{p = 1, \text{SNR} = 40\}$, respectively. 
 NA indicates that the algorithm did not return a solution with global optimality guarantee within 3600 seconds although the solution returned by CPLEX could be optimal. 
\begin{center}
	\begin{table*}[ht!]
	\begin{center}
			\caption{Running time in seconds for GFPI for different values of $r$, SNRs, 
			and purity~$p$. 
			The total number of data points is $n = n_1 \times r$. 
			An entry NA in the table means that GFPI took more than 1 hour to certify global optimality of the MIPs~\eqref{mixedinteger_fr_outlier} for all facets, we used a timelimit for CPLEX of $3600/r$ seconds. 
			}
			\label{tab:runningtime}  
			\small\addtolength{\tabcolsep}{-1pt}
			\begin{tabular}{c||ccccc}
				\hline
				$p = 1$, SNR = $\infty$ & $n_1 = 10$ & $n_1 = 30$ & $n_1 = 50$ & $n_1 = 100$ & $n_1 = 200$ \\\hline
				$r = 3$ & 0.75 & 0.73 & 2.81 & 44.53 & NA \\
				$r = 4$ & 0.61 & 5.48 & 44.66 & NA & NA \\
				$r = 5$ & 1.61 & 49.82 & NA & NA & NA\\
				$r = 7$ & 13.30 & NA & NA & NA & NA \\
				\hline
			\end{tabular} 
		\end{center}
	\begin{center}
			\small\addtolength{\tabcolsep}{-1pt}
			\begin{tabular}{c||ccccc}
				\hline
			   $p = 0.6$, SNR = $\infty$  & $n_1 = 10$ & $n_1 = 30$ & $n_1 = 50$ & $n_1 = 100$ & $n_1 = 200$ \\\hline
				$r = 3$ & 0.59 & 0.71 & 1.28 & 27.1 & 449.90  \\
				$r = 4$ & 0.55 & 3.40 & 29.46 & 609.18 & NA \\
				$r = 5$ & 1.35 & 39.76 & 653.19 & NA & NA \\
				$r = 7$ & 10.53 & NA & NA & NA & NA \\
				\hline
			\end{tabular} 
		\end{center}
		\begin{center}
			\small\addtolength{\tabcolsep}{-1pt}
			\begin{tabular}{c||ccccc}
				\hline
			$p = 1$, SNR = 40	& $n_1 = 10$ & $n_1 = 30$ & $n_1 = 50$ & $n_1 = 100$ & $n_1 = 200$ \\\hline
				$r = 3$ & 0.67 & 2.15 & 8.11 & 282.54  & NA \\
				$r = 4$ & 0.87 & 18.05 & 606.67 & NA & NA \\
				$r = 5$ & 3.37 & NA & NA & NA & NA \\
				$r = 7$ & 119.56 & NA & NA & NA & NA \\
				\hline
			\end{tabular} 
		\end{center}
		\begin{center}
			\small\addtolength{\tabcolsep}{-1pt}
			\begin{tabular}{c||ccccc}
				\hline
			$p = 0.6$, SNR = 40	& $n_1 = 10$ & $n_1 = 30$ & $n_1 = 50$ & $n_1 = 100$ & $n_1 = 200$ \\\hline
				r = 3 & 0.62 & 1.89 & 12.45 & 998.29 & NA \\
				r = 4 & 0.66 & 12.01 & 155.02 & NA & NA \\
				r = 5 & 2.94 & 516.09 & NA & NA & NA \\
				r = 7 & 103.89 & NA & NA & NA & NA \\
				\hline
			\end{tabular} 
		\end{center}
	\end{table*}
\end{center}

As $n$, $r$ and the noise level increase, the computational time increases rapidly. 
Even though CPLEX might require a lot of time to get optimality guarantees,  
it is usually able to find the optimal solution quite fast. In particular, for the noiseless case, for a problem with $n=1000$ (total number of points on all facets), each facet for $r=3$, $r=5$ and $r=7$ can be recovered accurately in less than 2, 2, and 10 seconds, respectively, although providing a global optimality certificate might take more than one hour, as shown in  Table~\ref{tab:runningtime}.

\subsection{Effect of the parameters $\eta$, $\lambda$ and $\gamma$ of GFPI} \label{app:param}

In Section \ref{discuss}, we claimed that the values for the parameters in GFPI can be selected 
via trial and error by keeping the solution with the largest number of data points on its facets. 
Figure~\ref{margin_vs_lambda} illustrates this observation for the parameters $\eta$ and $\lambda$. We use the previously introduced the synthetic data sets with $n_1 = 30$ and $n_2 = 10$. We set $m=r=3$, $SNR = 40$ and $\gamma = 0.1$.  We consider $\eta = \{0.1,0.3,0.5, 0.7, 0.9\}$ and $\lambda = \{0.5,1.5,2.5,3.5,4.5\}$. 

Figure~\ref{margin_vs_lambda} reports the average 
ERR metric (left), and average total number of points on all facets (right) for these values of $\eta$ and $\lambda$ over 10 trials.  
\begin{figure*}[htb]
	\begin{minipage}[b]{0.45\linewidth}
		\centering
		\centerline{\includegraphics[width=7cm]{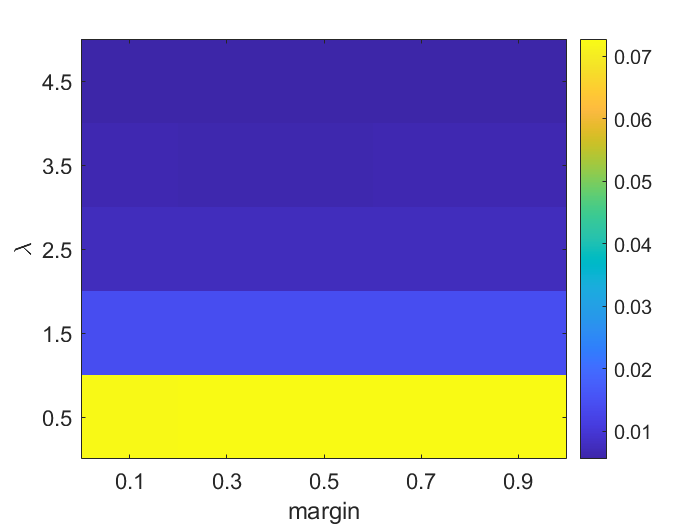}}
		\centerline{(a) ERR metric.}\medskip
	\end{minipage}
	\hfill
	\begin{minipage}[b]{0.45\linewidth}
		\centering
		\centerline{\includegraphics[width=7cm]{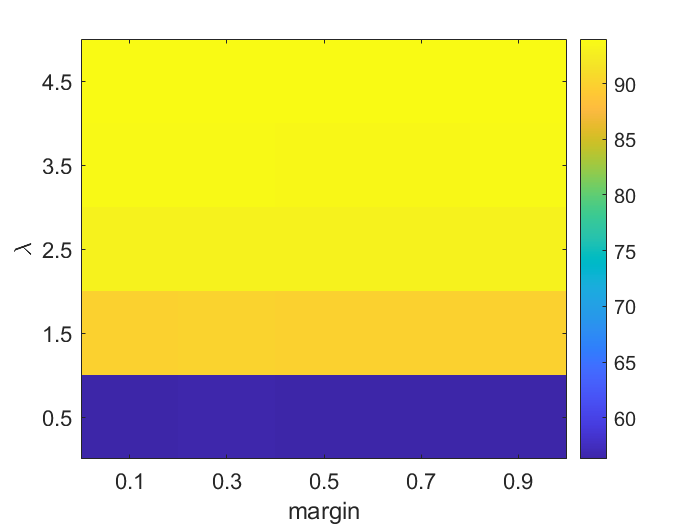}}
		\centerline{(b) Total number of points on all facets.}\medskip
	\end{minipage}
	\caption{The effect of different values for $\lambda$ and $\eta$ on relative error (a) and total number of points on facets (b). The correspondence between these two plots illustrates that the total number of points on the identified facets can be used a valid metric to  select the parameters. 
	\label{margin_vs_lambda}
	}
\end{figure*}

We observe that, as expected, there is a one-to-one correspondence between relative error and total number of points on the facets: low values of relative error are associated with high values of number of points on facets. 
We have performed  the same experiment for the case $p = 1$ and similar observations apply. 
Note that based on discussion in Section~\ref{sec:cutfacets}, 
 the margin is adapted automatically in our implementation. 
 Hence, the algorithm does not show much sensitivity to the value of the margin. \\

As stated in the paper, the parameter $\gamma$ has a  physical interpretation and can be determined based on the estimation of the noise. However, with no prior knowledge on the level of noise, setting this parameter might not be trivial. The question is therefore: 
\emph{When is the performance of GFPI sensitive to the value of $\gamma$?} 
Let us consider two examples for two different values of purity, namely: $p=0.55$ and $p=0.9$. 
For both cases, we generate samples using our generative model with $n_1=30, n_2=10$, $m=r=3$, SNR $=40$.  
We consider $\gamma = 0.1, 0.2$. 
Figure~\ref{gamma} shows the results.  We observe that, 
for a low purity ($p=0.55$), 
noisy data can lead to ambiguity in recovering the enclosing polytope, and the value of $\gamma$ may play a crucial role. 
A larger value of $\gamma$ means a larger safety gap around facets, and hence noisier facets (which are not facets of the ground truth simplex) get selected for $\gamma = 0.2$. In fact, for $\gamma = 0.2$, the data points on two different facets end up being associated with a single facet. 
For larger purity values ($p=0.9$), the performance is less sensitive to the value of $\gamma$.
\begin{figure*}[h]
	\begin{minipage}[b]{0.5\linewidth}
		\centering
		\centerline{\includegraphics[width=\textwidth]{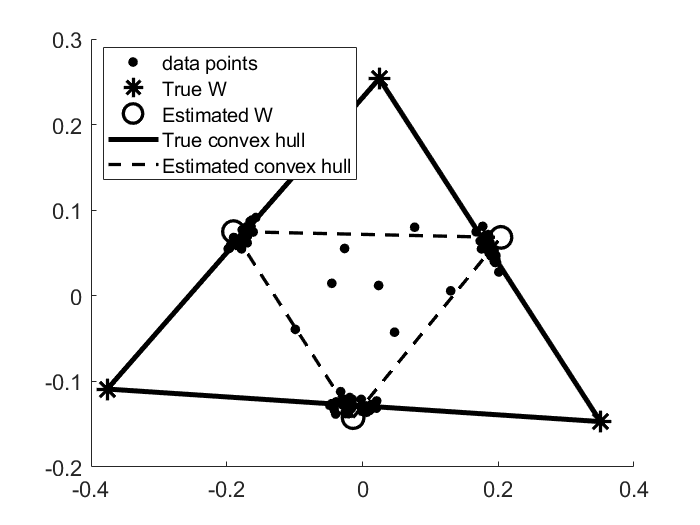}}
		\centerline{(a) purity = 0.55, $\gamma$ = 0.2}\medskip
	\end{minipage}
	\hfill
	\begin{minipage}[b]{0.5\linewidth}
		\centering
		\centerline{\includegraphics[width=\textwidth]{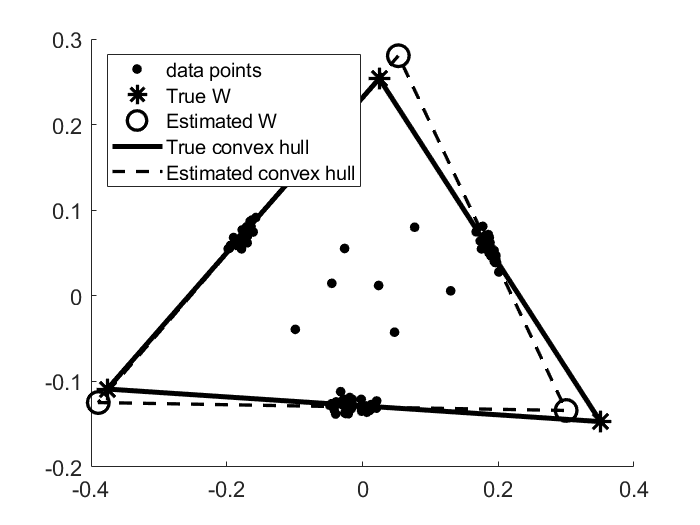}}
		\centerline{(b) purity = 0.55, $\gamma$ = 0.1}\medskip
	\end{minipage}
	\hfill
	\begin{minipage}[b]{0.5\linewidth}
		\centering
		\centerline{\includegraphics[width=\textwidth]{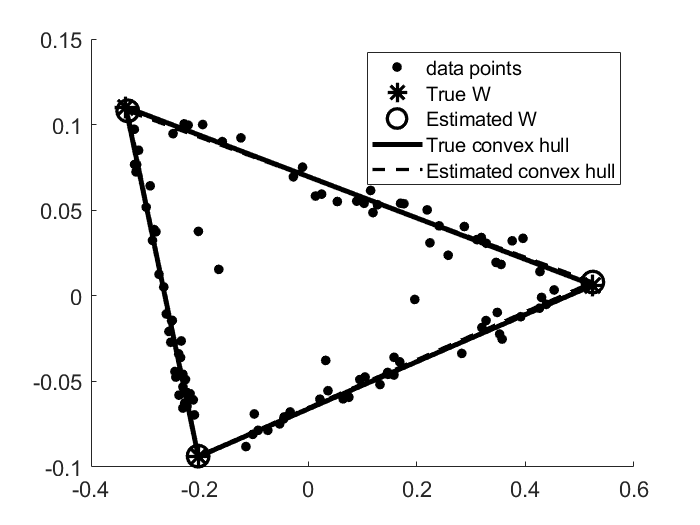}}
		\centerline{(c) purity = 0.9, $\gamma$ = 0.2}\medskip
	\end{minipage}
	\hfill
	\begin{minipage}[b]{0.5\linewidth}
		\centering
		\centerline{\includegraphics[width=\textwidth]{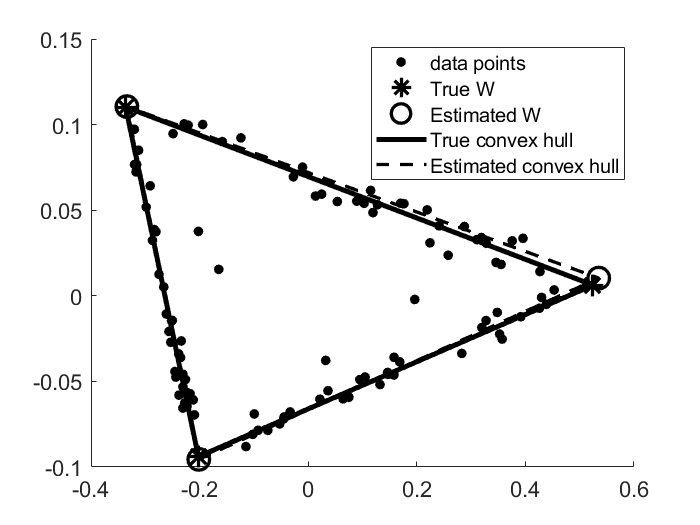}}
		\centerline{(d) purity = 0.9 \& $\gamma$ = 0.1}\medskip
	\end{minipage}
	\caption{Performance of GFPI depending on the purity and the value of $\gamma$, in noisy conditions.} 
	\label{gamma}
\end{figure*}

\end{document}